\documentclass{article}

\PassOptionsToPackage{numbers, compress}{natbib}
\PassOptionsToPackage{colorlinks=true, linkcolor=blue, pdfborder={0 0 0}}{hyperref}
\usepackage[final]{neurips_2022}

\usepackage[utf8]{inputenc} % allow utf-8 input
\usepackage[T1]{fontenc}    % use 8-bit T1 fonts
\usepackage{hyperref}       % hyperlinks
\usepackage{url}            % simple URL typesetting
\usepackage{booktabs}       % professional-quality tables
\usepackage{amsfonts}       % blackboard math symbols
\usepackage{nicefrac}       % compact symbols for 1/2, etc.
\usepackage{microtype}      % microtypography
\usepackage[usenames,dvipsnames]{xcolor}         % colors

\usepackage{xspace}

\usepackage{mathtools}
\usepackage{amsmath}
\usepackage{amssymb}
\usepackage{amsthm}
\usepackage{thmtools}
\usepackage[ruled,linesnumbered]{algorithm2e}

\usepackage{subfig}
\usepackage{float}
\usepackage{wrapfig}
\restylefloat{figure}

\newcommand*{\argmax}{\operatornamewithlimits{argmax}}
\newcommand*{\argmin}{\operatornamewithlimits{argmin}}

\newcommand*{\field}[1]{\mathbb{\MakeUppercase{#1}}}		% scalar field
\newcommand*{\set}[1]{{\mathcal{\MakeUppercase{#1}}}}			% set symbol
\newcommand*{\norm}[1]{\lVert #1 \rVert}	% norm of a vector

\newcommand*{\inner}[1]{\langle #1 \rangle}	% inner product between vectors, e.g. \inner{v,w}
\newcommand*{\collection}[1]{{\mathfrak{\MakeUppercase{#1}}}} % collection of sets

\newcommand*{\functional}[1]{{\MakeUppercase{#1}}}

\newcommand*{\R}{\field{R}} % field of real numbers, or the real line
\renewcommand{\vec}[1]{{\boldsymbol{\mathbf{#1}}}}
\newcommand*{\mat}[1]{\vec{\MakeUppercase{#1}}}
\newcommand*{\eye}{\mat{I}}							% the identity matrix
\newcommand*{\transpose}{\mathsf{T}}
\newcommand*{\dataset}{\set{D}} 				% observations dataset
\newcommand*{\observation}{y} 					% noisy output observation
\newcommand*{\observations}{\vec{\observation}}
\newcommand*{\obsNoise}{\epsilon}
\newcommand*{\gp}{\mathcal{GP}}
\newcommand*{\gpMean}{\mu}
\newcommand*{\gpMeanFunction}{m}
\newcommand*{\mutinfo}[1]{I(#1)}				% mutual information
\newcommand*{\gpfunction}{h}
\newcommand*{\mig}{\xi}
\newcommand*{\gpnoise}{\nu}

\newcommand*{\slocation}{x} % single location coordinate value
\newcommand*{\location}{\vec{\slocation}} % location, vector of coordinates
\newcommand*{\locDomain}{\set{X}} % location domain, decision set
\newcommand*{\dimension}{d}
\newcommand*{\locDim}{\dimension}

\newcommand*{\expectation}{\mathbb{E}}

\newcommand*{\Pspace}{\set{P}} 					% set of probability measures

\newcommand*{\kl}[2]{D_\mathrm{KL}(#1||#2)}
\newcommand*{\normal}{\mathcal{N}}					% normal distribution
\newcommand*{\uniform}{\mathcal{U}}

\newcommand*{\indicator}{\mathbb{I}} 				% indicator function
\newcommand*{\diff}{{\mathop{}\operatorname{d}}}

\newcommand*{\filtration}{\collection{F}}

\newcommand*{\cdf}{\Phi}
\newcommand*{\variational}{q}

\newcommand*{\Hspace}{\set{H}} 						% Hilbert space
\newcommand*{\regFactor}{\lambda}

\newcommand*{\af}{a} 						% acquisition function
\newcommand*{\regret}{r}					% instant regret
\newcommand*{\Regret}{\functional{R}}					% cumulative regret
\newcommand*{\uncertain}[1]{\bar{#1}}
\newcommand*{\eregret}{\uncertain{\regret}}	% expected regret
\newcommand*{\eRegret}{\uncertain{\Regret}}	% cumulative expected regret

\newcommand*{\objective}{f}

\newcommand*{\quantile}{\tau}
\newcommand*{\percentile}{\gamma}
\newcommand*{\lowerpdf}{\ell}
\newcommand*{\greaterpdf}{g}
\newcommand*{\classifier}{\pi}
\newcommand*{\clabel}{z}
\newcommand*{\ratio}{\rho}
\newcommand*{\loss}{\mathcal{L}}

\newcommand*{\rate}{\alpha}
\newcommand*{\sgrad}{\vec{\zeta}}

\newcommand*{\iterIdx}{t}

\newcommand*{\primIdx}{i}	% primary index
\newcommand*{\nIterations}{T}
\newcommand*{\nObs}{N}

\newcommand*{\nSamples}{M}

\newcommand*{\nFeatures}{F}

\newcommand*{\parameters}{\vec{\theta}} % parameters for a single input location's distribution
\newcommand*{\Lipschitz}{L}	% Lipschitz constant
\newcommand*{\bound}{b}		% RKHS norm bound for theoretical results

\newcommand*{\normalisation}{\eta}

\newcommand*{\anyscalar}{s}
\newcommand*{\anyfunction}{h}

\newcommand*{\iid}{i.i.d.\xspace}

\providecommand{\varitem}{} % to keep LaTeX quiet


\declaretheorem[name=Theorem]{theorem}

\declaretheorem[name=Lemma]{lemma}

\title{Batch Bayesian optimisation via density-ratio estimation with guarantees}

\author{%
	Rafael Oliveira$^{1,2}$\thanks{Corresponding author.}\\
	\texttt{rafael.oliveira@sydney.edu.au}\\
	 \And
	 Louis C. Tiao$^3$ \\
	 \texttt{louis.tiao@sydney.edu.au}\\
	 \AND
	 Fabio Ramos$^{3,4}$\\
	 \texttt{fabio.ramos@sydney.edu.au}\\
	 \And
	 \textnormal{$^1$Brain and Mind Centre, the University of Sydney, Australia}\\
	 $^2$ARC Training Centre in Data Analytics for Resources and Environments, Australia\\
	 $^3$School of Computer Science, the University of Sydney, Australia\\
	 $^4$NVIDIA, USA
}

\begin{document}

	\maketitle	
	
	\begin{abstract}
	Bayesian optimisation (BO) algorithms have shown remarkable success in applications involving expensive black-box functions. Traditionally BO has been set as a sequential decision-making process which estimates the utility of query points via an acquisition function and a prior over functions, such as a Gaussian process. Recently, however, a reformulation of BO via density-ratio estimation (BORE) allowed reinterpreting the acquisition function as a probabilistic binary classifier, removing the need for an explicit prior over functions and increasing scalability. In this paper, we present a theoretical analysis of BORE's regret and an extension of the algorithm with improved uncertainty estimates. We also show that BORE can be naturally extended to a batch optimisation setting by recasting the problem as approximate Bayesian inference. The resulting algorithms come equipped with theoretical performance guarantees and are assessed against other batch and sequential BO baselines in a series of experiments.
\end{abstract}

\section{Introduction}
Bayesian optimisation (BO) algorithms provide flexible black-box optimisers for problems involving functions which are noisy or expensive to evaluate \citep{Shahriari2016}. Typical BO approaches place a probabilistic model over the objective function which is updated with every new observation in a sequential decision-making process. Most methods are based on Gaussian process (GP) surrogates \citep{Rasmussen2006}, which provide closed-form analytic expressions for the model's posterior distribution and allow for a number of theoretical performance guarantees \citep{Srinivas2010, Bull2011, Wang2018meta}. However, GP surrogates have a number of limitations, such as not easily scaling to high-dimensional domains, high computational complexity and requiring a careful choice of covariance function and hyper-parameters \citep{Rasmussen2006}. Non-GP-based BO methods have also been proposed in the literature, such as BO methods based on neural networks \citep{Snoek2015, Springenberg2016} and random forests \citep{Hutter2011smac} regression models.

As an alternative to improving the model, \citet{tiao2021bore} focus on the acquisition function, which in BO frameworks represents the guide that takes the model predictions into account. They show that one can derive the acquisition function directly without an implicit model by reinterpreting the expected improvement \citep{Bull2011, Shahriari2016} via a density-ratio estimation problem. Applying this perspective, the acquisition function can then be derived as a classification model, which can be represented by flexible parametric models, such as deep neural networks, and efficiently trained via stochastic gradient descent. The resulting method, called \emph{Bayesian optimisation via density-ratio estimation} (BORE) is then shown to outperform a variety of traditional GP-based and non-GP baselines.

Despite the significant performance gains, BORE has only been applied to a sequential setting and not much is known about the method's theoretical guarantees. Batch BO methods have the potential to speed up optimisation in settings where multiple queries to the objective function can be evaluated simultaneously \citep{Snoek2012, Gonzalez2016batch, Wang2017batch, Wilson2018}. Given its flexibility to apply models which can scale to large datasets, it is therefore a natural question as to whether BORE can be readily extended to the batch setting in a computationally efficient way.

In this paper, we extend the BORE framework to the batch setting and analyse its theoretical performance. To derive theoretical guarantees, we first show that the original BORE can be improved by accounting for uncertainty in the classifier's predictions. We then propose a novel method, called BORE++, which uses an upper confidence bound over the classifier's predictions as its acquisition function. The method comes equipped with guarantees in the probabilistic least-squares setting. We provide extensions for both BORE and BORE++ to the batch setting. Lastly, we present experimental results demonstrating the performance of the proposed algorithms in practical optimisation problems. 

\section{Background}
We consider a global optimisation problem over a compact search space $\locDomain\subset\R^\locDim$ of the form:
\begin{equation}
	\location^* \in \argmin_{\location\in\locDomain} \objective(\location)\,,
	\label{eq:problem}
\end{equation}
where  $\objective:\locDomain\to\R$ is assumed to be a black-box objective function, i.e., we have no access to gradients nor analytic formulations of it. In addition, we are only allowed to run up to $\nIterations$ rounds of function evaluations, where we might collect single points or batches of observations $\observation_\iterIdx := \objective(\location_\iterIdx) + \obsNoise_\iterIdx$, which are corrupted by additive noise $\obsNoise_\iterIdx$, for $\iterIdx\in\{1,\dots,\nIterations\}$.

\subsection{Bayesian optimisation}
Bayesian optimisation (BO) algorithms approach the problem in \autoref{eq:problem} via sequential decision making \citep{Shahriari2016}. At each iteration, BO selects a query point by maximising an acquisition function $\af$:
\begin{equation}
	\location_\iterIdx \in \argmax_{\location\in\locDomain} \af(\location|\dataset_{\iterIdx-1})
\end{equation}
The acquisition function encodes information provided by the observations collected so far $\dataset_\iterIdx := \{\location_i, \observation_i\}_{i=1}^{\iterIdx-1}$ using a probabilistic model over $\objective$, typically a Gaussian process (GP) \citep{Rasmussen2006}, conditioned on the data. After collecting an observation $\observation_\iterIdx$, the dataset is updated with the new query-observation pair $\dataset_\iterIdx := \dataset_{\iterIdx-1} \cup {\location_\iterIdx, \observation_\iterIdx}$. This process then repeats for a given number of iterations $\nIterations$.

%\subsection{The expected improvement algorithm}
%A popular acquisition function in the BO literature is the expected improvement (EI) \citep{Bull2011}. At iteration $\iterIdx \geq 1$, one can define $\quantile := \min_{\primIdx<\iterIdx} \observation_\iterIdx$ as an incumbent target. EI is then defined as:
%\begin{equation}
%	\af_{\operatorname{EI}}(\location|\dataset_{\iterIdx-1}) := \expectation[\max\left\{0,\quantile - f(\location)\right\}|\dataset_{\iterIdx-1}]~.
%	\label{eq:ei}
%\end{equation}
%In the case of a GP prior on $f|\dataset_{\iterIdx-1}\sim\gp(\gpMean_{\iterIdx-1}, k_{\iterIdx-1})$, the EI is analytic and given by $\af_{\operatorname{EI}}(\location|\dataset_{\iterIdx-1}) =	(\quantile - \gpMean_{\iterIdx-1}(\location))\Psi(s_{\iterIdx-1}) + \sigma_{\iterIdx-1}(\location)\psi(s_{\iterIdx-1})$, where $s_{\iterIdx-1} := \frac{\quantile - \gpMean_{\iterIdx-1}(\location)}{\sigma_{\iterIdx-1}(\location)}$, if $\sigma^2_{\iterIdx-1}(\location) := k_{\iterIdx-1}(\location,\location) > 0$. For points $\location\in\locDomain$ where $\sigma_{\iterIdx-1}(\location) = 0$, we set $\af_{\operatorname{EI}}(\location|\dataset_{\iterIdx-1}) := 0$ by convention. Here $\Psi(s)$ and $\psi(s)$ denote, respectively, the cumulative distribution function (CDF) and the probability density function (PDF) of the standard normal distribution evaluated at $s\in\R$.

\subsection{Bayesian optimisation via density-ratio estimation (BORE)}
\label{sec:bore}
The expected improvement (EI) \citep{Bull2011, Jones1998} is a popular acquisition function in the BO literature and the basis for many BO algorithms. At each iteration $\iterIdx \geq 1$, one can define $\quantile := \min_{\primIdx<\iterIdx} \observation_\iterIdx$ as an incumbent target. EI is then defined as:
\begin{equation}
	\af_{\operatorname{EI}}(\location|\dataset_{\iterIdx-1}) := \expectation[\max\left\{0,\quantile - f(\location)\right\}|\dataset_{\iterIdx-1}]~.
	\label{eq:ei}
\end{equation}
In the case of a GP prior on $f|\dataset_{\iterIdx-1}\sim\gp(\gpMean_{\iterIdx-1}, k_{\iterIdx-1})$, the EI is available in closed form as a function of the GP posterior. However, the EI may be reformulated without the need for a prior.

Under mild assumptions, \citet{bergstra2011algorithms} showed that the EI can be formulated as a density ratio between two probability distributions. Let $\lowerpdf(\location) := p(\location|\observation\leq\quantile)$ represent the probability density over  $\location\in\locDomain$ conditioned on the observation $\observation$ being below a threshold $\quantile\in\R$. Conversely, let $\greaterpdf(\location) := p(\location|\observation > \quantile)$. For $\percentile\in[0,1]$, the $\percentile$-relative density ratio between these two densities is:
\begin{equation}
	\ratio_\percentile(\location) := \frac{\lowerpdf(\location)}{\percentile\lowerpdf(\location) + (1-\percentile)\greaterpdf(\location)}\,, \quad \location\in\locDomain\,,
\end{equation}
noting that $\percentile = 0$ leads to the ordinary probability density ratio definition, $\ratio_0(\location) = \nicefrac{\lowerpdf(\location)}{\greaterpdf(\location)}$. Now if we choose $\quantile := \cdf^{-1}(\percentile)$, where $\cdf(\anyscalar) := p(\observation \leq \anyscalar)$ represents the cumulative distribution function of the marginal distribution of observations,\footnote{Note that $p(\observation \leq \anyscalar) = \int_{\locDomain} p(\observation\leq\anyscalar | \location)p(\location)\diff\location$, where we may assume $p(\location)$ uniform.} for $\anyscalar\in\R$, and then replace $\quantile$ in \autoref{eq:ei}, \citet{bergstra2011algorithms} have shown that\footnote{\citet{bergstra2011algorithms} and \citet{tiao2021bore} also rely on the mild assumption that $p(\location|\observation) \approx \lowerpdf(\location)$ for all $\observation\leq\quantile$.} $\af_{\operatorname{EI}}(\location) \propto \ratio_\percentile(\location)$, for $\location\in\locDomain$. Based on this fact, \citet{tiao2021bore} showed:
\begin{equation}
	\af_{\operatorname{EI}}(\location) \propto \ratio_\percentile(\location) = \percentile^{-1}\classifier(\location), \quad \location\in\locDomain,
\end{equation}
where $\classifier(\location) := p(\observation\leq \quantile |\location)$ can be approximated by a probabilistic classifier trained with a proper scoring rule, such as the binary cross-entropy loss:
\begin{equation}
	\loss_\iterIdx[\classifier] := \sum_{i=1}^\iterIdx \clabel_i\log \classifier(\location_i) + (1-\clabel_i)\log(1-\classifier(\location_i))\,.
\end{equation}
Other examples of proper scoring rules include the least-squares loss, which leads to probabilistic least-squares classifiers \citep{Sugiyama2012}, and the zero-one loss. We refer the reader to \citet{gneiting2007strictly} for a review and theoretical analysis on this topic.

\begin{algorithm}[t]
	\caption{BORE}
	\label{alg:bore}
	\DontPrintSemicolon
	\For{$\iterIdx\in\{1,\dots,\nIterations\}$}
	{
		$\quantile := \hat\cdf_{\iterIdx-1}^{-1}(\percentile)$\;
		$\clabel_i := \indicator[\observation_i \leq \quantile], \quad i\in \{1,\dots,\iterIdx-1\}$\;
		$\tilde\dataset_{\iterIdx-1} := \{\location_i, \clabel_i\}_{i=1}^{\iterIdx-1}$\;
		$\hat\classifier_\iterIdx \in \argmin_\classifier \loss[\classifier|\tilde\dataset_{\iterIdx-1}]$\;
		$\location_\iterIdx \in \argmax_{\location\in\locDomain} \hat\classifier_{\iterIdx-1}(\location)$\;
		$\observation_\iterIdx := \objective(\location_\iterIdx) + \obsNoise_\iterIdx$\;
	}
\end{algorithm}

BORE is summarised in \autoref{alg:bore}. As seen, the marginal observations distribution CDF $\cdf(\anyscalar) := p(\observation\leq \anyscalar)$ is replaced by the empirical approximation $\hat\cdf_\iterIdx(\anyscalar) := \frac{1}{\iterIdx}\sum_{i=1}^\iterIdx \indicator[\observation_i\leq \anyscalar]$ and its corresponding quantile function $\hat\cdf_\iterIdx^{-1}$. At each iteration, observations are labelled according to the estimated $\percentile$th quantile $\quantile$, and a classifier $\hat\classifier_\iterIdx$ is trained by minimising the loss $\loss[\classifier|\tilde\dataset_\iterIdx]$ over the data points $\tilde\dataset_\iterIdx$. A query point $\location_\iterIdx$ is chosen by maximising the classifier's probabilities, which in our case corresponds to maximising the expected improvement. A new observation is collected, and the algorithm continues running up to a given number of iterations $\nIterations$. As demonstrated, no explicit probabilistic model for $\objective$ is needed, only a classifier, which can be efficiently trained via, e.g., stochastic gradient descent.

\section{Analysis of the BORE framework}
In this section, we analyse limitations of the BORE framework in modelling uncertainty and analyse its effects on the algorithm's performance. As presented in \autoref{sec:bore}, at each iteration $\iterIdx\geq 1$, the original BORE framework trains a probabilistic classifier $\hat\classifier_\iterIdx(\location)$ to approximate $p(\observation\leq\quantile|\location)$, where $\quantile$ denotes the $\percentile$th quantile of the marginal observations distribution, i.e., $p(\observation\leq\quantile) = \gamma$. This approach leads to a maximum likelihood estimate for the classifier $\hat\classifier$, which may not properly account for the uncertainty in the classifier's approximation. %, leading to saturation for the class probability $p(\observation\leq\quantile|\location) \approx 1$ around previously observed points $\location_\iterIdx$ where $\observation_\iterIdx \leq \quantile$. %TODO: Add figure illustrating phenomena.

Since BORE is based on probabilistic classifiers, instead of regression models as in traditional BO frameworks \citep{Shahriari2016}, a natural first question to ask is whether a classifier can guide it to the global optimum of the objective function. The following lemma answers this question and is a basis for our analysis.

\begin{lemma}
	\label{thr:classifier-optimum}
	Let $\objective:\locDomain\to\R$ be a continuous function over a compact space $\locDomain$. Assume that, for any $\location\in\locDomain$, we observe $\observation = \objective(\location) + \obsNoise$, where $\obsNoise$ is \iid noise with a strictly monotonic cumulative distribution function $\cdf_\obsNoise:\R\to[0,1]$. Then, for any $\quantile\in\R$, we have:
	\begin{equation}
		\argmax_{\location\in\locDomain} p(\observation \leq \quantile|\location, \objective) = \argmin_{\location\in\locDomain} f(\location)\,.
	\end{equation}
\end{lemma}
\begin{proof}
	As the observation noise CDF is monotonic, by basic properties of the $\argmax$, we have:
	\begin{equation}
		\begin{split}
			\argmax_{\location\in\locDomain} p(\observation \leq \quantile|\location, \objective)
			= \argmax_{\location\in\locDomain} \cdf_\obsNoise(\quantile-\objective(\location)) = \argmin_{\location\in\locDomain} \objective(\location)\,,
		\end{split}
	\end{equation}
	which concludes the proof.
\end{proof}

According to this lemma, maximising class probabilities is equivalent to optimising the objective function when the classifier is optimal, i.e., it has perfect knowledge of $\objective$. This result holds for any given threshold $\quantile\in\R$. We only make a mild assumption on the CDF of the observation noise $\cdf_\obsNoise$, which is satisfied for any probability distribution with support covering the real line (e.g. Gaussian, Student-T, Cauchy, etc.).\footnote{This result could also be easily extended to distributions with bounded support as long as their CDF is monotonic within it. However, we keep the support as $\R$ for simplicity, and the extension is left for future work.}

To analyse BORE's optimisation performance, we will aim to bound the algorithm's instant regret:
\begin{equation}
	\regret_\iterIdx := \objective(\location_\iterIdx) - \objective(\location^*), \quad \iterIdx\geq 1,
\end{equation}
and its cumulative version $\Regret_\nIterations := \sum_{\iterIdx=1}^\nIterations \regret_\iterIdx$. Sub-linear bounds on $\Regret_\nIterations$ lead to a no-regret algorithm, since $\lim_{\nIterations\to\infty}\frac{\Regret_\nIterations}{\nIterations} = 0$ and $\min_{\iterIdx \leq \nIterations} \regret_\iterIdx \leq \frac{\Regret_\nIterations}{\nIterations}$.

Assuming that there is an optimal classifier $\classifier^*:\locDomain\to[0,1]$, which is such that $\classifier^*(\location) = p(\observation \leq \quantile|\location, \objective)$, for a given $\quantile\in\R$, we can directly relate the classifier probabilities to the objective function $\objective$ values, since:
\begin{equation}
	\begin{split}
		\classifier^*(\location) = p(\observation\leq \quantile|\location, \objective) = \cdf_\obsNoise(\quantile-\objective(\location))		\quad\therefore\quad \objective(\location) = \quantile - \cdf_\obsNoise^{-1}(\classifier^*(\location))\,.
	\end{split}
\end{equation}
The existence of the inverse $\cdf_\obsNoise^{-1}$ is ensured by the strict monotonicity assumption on $\cdf_\obsNoise$ in \autoref{thr:classifier-optimum}. Under this observation, the algorithm's regret at any iteration $\iterIdx\geq 1$ can be bounded in terms of classifier probabilities:
\begin{equation}
	\begin{split}
		\regret_\iterIdx &= \objective(\location_\iterIdx) - \objective(\location^*) = \cdf_\obsNoise^{-1}(\classifier^*(\location^*)) - \cdf_\obsNoise^{-1}(\classifier^*(\location_\iterIdx)) \leq \Lipschitz_\obsNoise (\classifier^*(\location^*) - \classifier^*(\location_\iterIdx))\,,
	\end{split}
	\label{eq:optimal-classifier-regret}
\end{equation}
where $\Lipschitz_\obsNoise$ is any Lipschitz constant for $\cdf_\obsNoise^{-1}$, which exists since $\locDomain$ is compact. Therefore, we should be able to bound BORE's regret by analysing the approximation error for $\hat\classifier_\iterIdx$ at each iteration $\iterIdx\geq 1$. 

Although approximation guarantees for classification algorithms under \iid data settings are well known \citep{Barron1994approximation}, %TODO: Cite
each observation in BORE depends on the previous ones via the acquisition function. This process is also not necessarily stationary, so that we cannot apply known results for classifiers under stationary processes \citep{Steinwart2009}. In the next section, we consider a particular setting for learning a classifier which allows us to bound the prediction error under BORE's data-generating process.

\subsection{Probabilistic least-squares classifiers}
\label{sec:pls}
We consider the case of probabilistic least-squares (PLS) classifiers \citep{Selten1998, Suykens1999}. In particular, we model a probabilistic classifier $\classifier:\locDomain\to[0,1]$ as an element of a reproducing kernel Hilbert space (RKHS) $\Hspace$ associated with a positive-definite kernel $k:\locDomain\times\locDomain\to\R$. A RKHS is a space of functions equipped with inner product $\inner{\cdot, \cdot}_k$ and norm $\norm{\cdot}_k := \sqrt{\inner{\cdot,\cdot}_k}$  \citep{Scholkopf2002}. For the purposes of this analysis, we will also assume that $k(\location,\location) \leq 1$, for all $\locDomain$.\footnote{This assumption can always be satisfied by proper scaling.}
%Since $\Hspace$ is a linear space, we know that not every $\classifier\in\Hspace$ yields a valid classifier, since $\classifier(\location)$ might be negative for all or some $\location\in\locDomain$. Yet we assume there is a $\classifier^* \in \Hspace$ such that $p(\observation\leq\quantile|\location,\objective) = \classifier^*(\location)$ for all $\location\in\locDomain$ and $\quantile\in\R$. 
This setting allows for both linear and non-parametric models. Gaussian assumptions on the function space would lead us to GP-based PLS classifiers \citep{Rasmussen2006}, but we are not restricted by Gaussianity in our analysis. If the kernel $k$ is universal, as $\cdf_\obsNoise$ is injective, we can also see that the RKHS assumption allows for modelling any continuous function.

For a given $\quantile\in\R$, a PLS classifier is obtained by minimising the regularised squared-error loss:
\begin{equation}
	\hat\classifier_\iterIdx \in \argmin_{\classifier\in\Hspace} \sum_{i=1}^\iterIdx (\clabel_i - \classifier(\location_i))^2 + \regFactor\norm{\classifier}_k^2\,,\quad \iterIdx\geq 1,
	\label{eq:pls-problem}
\end{equation}
where $\regFactor > 0$ is a given regularisation factor and $\clabel_i := \indicator[\observation_i\leq \quantile] \in \{0,1\}$. In the RKHS case, the solution to the problem above is available in closed form \citep{Abbasi-Yadkori2010, Sugiyama2012} as:
\begin{equation}
	\hat\classifier_\iterIdx(\location) = \vec k_\iterIdx(\location)^\transpose (\mat k_\iterIdx + \regFactor\eye)^{-1}\vec\clabel_\iterIdx\,, \quad \location\in\locDomain,\, \iterIdx\geq 1,
	\label{eq:pls-estimator}
\end{equation}
where $\vec k_\iterIdx(\location) := [k(\location, \location_1), \dots, k(\location,\location_\iterIdx)]^\transpose \in \R^\iterIdx$, $\mat k_\iterIdx := [k(\location_i, \location_j)]_{i,j=1}^\iterIdx \in \R^{\iterIdx\times\iterIdx}$ and $\vec\clabel_\iterIdx := [\clabel_1,\dots,\clabel_\iterIdx]^\transpose \in \R^\iterIdx$. This PLS approximation  may not yield a valid classifier, since it is possible that $\hat\classifier_\iterIdx(\location) \notin [0,1]$ for some $\location\in\locDomain$. However, it allows us to place a confidence interval on the optimal classifier's prediction, as presented in the following theorem, which is based on theoretical results from the online learning literature \citep{Abbasi-Yadkori2012, Durand2018}. Our proofs can be found in the supplement.
\begin{theorem}
	\label{thr:pls-ucb}
	Given $\quantile\in\R$, assume $\classifier(\location) := \cdf_\obsNoise(\quantile-\objective(\location))$ is such that $\classifier\in\Hspace$, {and $\norm{\classifier}_k \leq \bound$}. Let $\{\location_\iterIdx\}_{\iterIdx=1}^\infty$ be a $\locDomain$-valued discrete-time stochastic process predictable with respect to the filtration $\{\filtration_\iterIdx\}_{\iterIdx=0}^\infty$. Let $\{\clabel_\iterIdx\}_{\iterIdx=1}^\infty$ be a real-valued stochastic process such that $\gpnoise_\iterIdx := \clabel_\iterIdx - \classifier(\location_\iterIdx)$ is $1$-sub-Gaussian conditionally on $\filtration_{\iterIdx-1}$, for all $\iterIdx \geq 1$. Then, for any $\delta\in(0,1)$, with probability at least $1-\delta$, we have that:
	\begin{equation}
		\forall\location\in\locDomain,\quad |\classifier(\location) - \hat\classifier_\iterIdx(\location)| \leq \beta_\iterIdx(\delta) \sigma_\iterIdx(\location), \quad \forall\iterIdx\geq 1\,,
	\end{equation}
	where ${\beta_\iterIdx(\delta) := \bound + \sqrt{2\regFactor^{-1}\log(|\eye+\regFactor^{-1}\mat k_\iterIdx|^{1/2}/\delta)}}$, with $|\mat A|$ denoting the determinant of matrix $\mat A$, and $\sigma_\iterIdx^2(\location) := k(\location,\location) - \vec k_\iterIdx(\location)^\transpose (\mat k_\iterIdx + \regFactor\eye)^{-1}\vec k_\iterIdx(\location)\,, \quad \location\in\locDomain, \quad \iterIdx\geq 1.$
\end{theorem}

\subsection{Regret analysis for BORE}
We now consider BORE with a PLS classifier. For this analysis, we will assume an ideal setting where $\quantile$ is fixed, possibly corresponding to the true $\gamma$th quantile of the observations distribution. However, our results hold for any choice of $\quantile\in\R$ and can therefore be assumed to approximately hold for a varying $\quantile$ which is converging to a fixed value. In this setting, the algorithm's choices are: given by:
\begin{equation}
	\location_\iterIdx \in \argmax_{\location\in\locDomain} \hat\classifier_{\iterIdx-1}(\location)\,,
\end{equation}
where $\hat\classifier_\iterIdx$ is the estimator in \autoref{eq:pls-estimator}. we can then apply \autoref{thr:pls-ucb} to the classifier-based regret in \autoref{eq:optimal-classifier-regret} to obtain a regret bound. For this result, we will also need the following quantity:
\begin{equation}
	\mig_\nObs := \max_{\{\location_i\}_{i=1}^\nObs\subset\locDomain} \frac{1}{2}\log|\eye + \regFactor^{-1}\mat k_\nObs|\,, \quad \nObs\geq 1\,,
	\label{eq:mig}
\end{equation}
where the maximisation is taken over the discrete set of locations $\{\location_i\}_{i=1}^\nObs\subset\locDomain$ and $\mat k_\nObs := [k(\location_i, \location_j)]_{i,j=1}^\nObs$. This quantity denotes the maximum information gain of a Gaussian process model after $\nObs$ observations. We are now ready to state our theoretical result regarding BORE's regret.

\begin{theorem}
	\label{thr:bore-regret}
	Under the conditions in \autoref{thr:pls-ucb}, with probability at least $1-\delta$, $\delta\in(0,1)$, the instant regret of the BORE algorithm with a PLS classifier after $\nIterations\geq 1$ iterations is bounded by:
	\begin{equation}
		\regret_\iterIdx \leq \Lipschitz_\obsNoise \beta_{\iterIdx-1}(\delta)(\sigma_{\iterIdx-1}(\location_\iterIdx) + \sigma_{\iterIdx-1}(\location^*)),
	\end{equation}
	and the cumulative regret by:
	\begin{equation}
		\Regret_\nIterations \leq \Lipschitz_\obsNoise\beta_\nIterations(\delta)\left( \sqrt{4(\nIterations+2)\mig_\nIterations} + \sum_{\iterIdx=1}^\nIterations \sigma_{\iterIdx-1}(\location^*) \right)\,.
	\end{equation}
\end{theorem}

As \autoref{thr:bore-regret} shows, the regret of the BORE algorithm in the PLS setting is comprised of two components. The first term is related to the regret of a GP-UCB algorithm \citep[see][Thr. 3]{Chowdhury2017} and its known to grow sub-linearly for a few popular kernels, such as the squared exponential and the Mat\'ern class \citep{Srinivas2010, Vakili2021}. The second term, however, reflects the uncertainty of the algorithm around the optimum location $\location^*$. If the algorithm never samples at that location, this second summation might have a mostly linear growth, which will not lead to a vanishing regret. In fact, if we consider \autoref{eq:pls-estimator} and a RKHS with a translation-invariant kernel, we see that, as soon as an observation $\clabel_\iterIdx=1$ is collected at a location $\location_\iterIdx \neq \location^*$, that location will constitute the maximum of the classifier output. Then the algorithm would keep returning to that same location, missing opportunities to sample at $\location^*$.

It is worth noting that \autoref{thr:bore-regret} reflects the regret of BORE in an idealistic setting where the algorithm uses the optimal PLS estimator in the function space $\Hspace$. However, if we train a parametric classifier, such as a neural network, via gradient descent, the behaviour will not necessarily be the same, and the algorithm might still achieve a good performance. In the original BORE paper, for instance, a parametric classifier is trained by minimising the binary cross-entropy loss \citep{tiao2021bore} and leads to a successful performance in experiments. Neural network models trained via stochastic gradient descent are known to provide approximate samples of a posterior distribution \citep{Bardsley2014, Mandt2017}, instead of an optimal best-fit predictor, which might make BORE behave like Thompson sampling \citep{Russo2016} (see discussion in the appendix).
Nevertheless, \autoref{thr:bore-regret} still shows us that BORE may get stuck into local optima, which is not ideal for BO methods. In the next section, we present an extension of the BORE framework which addresses this shortcoming.

\section{BORE++: improved uncertainty estimates}
% \label{sec:bore-ucb}
%  and then propose a solution to this problem via a quantile-based formulation of the classifier.

% % The algorithm may then over exploit the search space and get stuck in local optima. To address this issue, we propose accounting for this form of epistemic uncertainty by applying a probabilistic model to update the classifier's parameters.

% \subsection{Analysis of the BORE algorithm}
% We start by presenting a few results concerning the original BORE framework. 

% \subsection{Improvements via uncertainty quantification}
As discussed in the previous section, the lack of uncertainty quantification in the estimation of the classifier for the original BORE might lead to sub-optimal performance. To address this shortcoming, we present an approach for uncertainty quantification in the BORE framework which leads to improvements in performance and theoretical optimality guarantees. Our approach is based on using an upper confidence bound (UCB) on the predicted class probabilities as the acquisition function for BORE. Due to its improved uncertainty estimates, we call this approach BORE++.

\subsection{Class-probability upper confidence bounds}
We propose replacing $\hat\classifier_\iterIdx$ in \autoref{alg:bore} by an upper confidence bound which is such that:
\begin{equation}
	\forall\iterIdx\geq 1,\quad \classifier^*(\location) \leq \classifier_{\iterIdx,\delta}(\location), \quad \forall \location\in\locDomain
	\label{eq:classifier-ucb}
\end{equation}
which with probability greater than $1-\delta$, given $\delta\in (0,1)$. Therefore, $\classifier_{\iterIdx,\delta}(\location)$ represents an upper quantile over the optimal class probability $\classifier^*(\location)$. BORE++ selects $\location_\iterIdx \in \argmax_{\location\in\locDomain} \classifier_{\iterIdx-1, \delta}(\location)$.

To derive an upper confidence bound on a classifier's predictions $\classifier(\location)$, we can take a few different approaches. For a parametric model $\classifier_\parameters$, a Bayesian model updating the posterior $p(\parameters|\dataset_\iterIdx)$ leads to a corresponding predictive distribution over $\classifier_\parameters(\location)$. This is the case of ensemble models \citep{Rokach2010ensemble}, for instance, where we approximate predictions $p(\observation\leq \quantile_\iterIdx|\location, \dataset_\iterIdx) \approx \frac{1}{\nSamples}\sum_{i=1}^\nSamples \classifier_{\parameters^i}(\location)$ with $\parameters^i \sim p(\parameters|\dataset_\iterIdx)$. Instead of using the expected class probability, however, BORE++ uses an (empirical) quantile approximation for $\classifier_{\iterIdx, \delta}$ to ensure \autoref{eq:classifier-ucb} holds. Bayesian neural networks \citep{Penny1999bnn}, random forests \citep{Amit1997forests}, dropout methods, etc. \citep{Polson2017}, also constitute valid approaches for predictive uncertainty estimation. An alternative approach is to place a non-parametric prior over $\classifier^*$, such as a Gaussian process model \citep{Rasmussen2006}, which allows for the modelling of uncertainty directly in the function space where $\classifier^*$ lies. In the next section, we present a concrete derivation of BORE++ for the PLS classifier setting which takes the non-parametric perspective and allows us to derive theoretical performance guarantees.

\subsection{BORE++ with PLS classifiers}
\label{sec:bore-ucb-regret}
{
	In the PLS setting, the result in \autoref{thr:pls-ucb} gives us a closed-form expression for a classifier upper confidence bound satisfying the condition in \autoref{eq:classifier-ucb}. Given $\delta\in(0,1)$, we set:
	\begin{equation}
		\classifier_{\iterIdx, \delta}(\location) :=\min(1, \max(0, \hat\classifier_\iterIdx(\location) + \beta_\iterIdx(\delta)\sigma_\iterIdx(\location)))\in[0,1]\,, \quad \location\in\locDomain\,,
		\label{eq:pls-ucb-classifier}
	\end{equation}
	where $\sigma_\iterIdx$ and $\beta_\iterIdx$ are set according to \autoref{thr:pls-ucb}. We then obtain the following result for BORE++.
}

\begin{theorem}
	\label{thr:pls-regret}	%TODO: Mention assumptions under which this result holds (e.g., \classifier in the RKHS)
	{Under the assumptions in \autoref{thr:pls-ucb}}, running the BORE++ algorithm with a PLS classifier $\classifier_{\iterIdx,\delta}$ as defined above yields, with probability at least $1-\delta$, an instant regret bound of:
	\begin{equation}
		\regret_\iterIdx \leq 2\Lipschitz_\obsNoise \beta_\iterIdx(\delta)\sigma_\iterIdx(\location)\,, \quad \forall \iterIdx\geq 1\,,
	\end{equation}
	and a cumulative regret bound after $\nIterations\geq 1$ iterations:
	\begin{equation}
		\begin{split}
			\Regret_\nIterations &\leq 4\Lipschitz_\obsNoise \beta_\nIterations(\delta)\sqrt{(\nIterations+2)\mig_\nIterations} \in \set{O}\left(\sqrt{\nIterations}(\bound\sqrt{\mig_\nIterations} + \mig_\nIterations)\right). %TODO: Define "\bound"
		\end{split}
	\end{equation}
\end{theorem}
According to \autoref{thr:pls-regret}, the regret of BORE++ vanishes if the maximum information gain $\mig_\nIterations$ grows sub-linearly, since $\lim_{\nIterations\to\infty}\frac{\Regret_\nIterations}{\nIterations} = 0$ and $\min_{\iterIdx \leq \nIterations} \regret_\iterIdx \leq \frac{\Regret_\nIterations}{\nIterations}$. Sub-linear growth is known to be achieved for popular kernels, such as the squared exponential, the Mat\'ern family and linear kernels \citep{Srinivas2010, Vakili2021}. This result also tells us that theoretically BORE++ performs no worse than GP-UCB since they share similar regret bounds \citep{Srinivas2010, Chowdhury2017}. However, in practice, the BORE++ framework offers a series of practical advantages over GP-UCB, such as no need for an explicit surrogate model, and a classifier which does not need to be a GP and can therefore be more flexible and scalable to high-dimensional problems and large amounts of data. The connection with GP-UCB, instead, brings us new insights into how the density-ratio BO algorithm can still share some of the well known guarantees of traditional BO methods.

%TODO: Comment on the distinction and common trait between GP-UCB and BORE++. The latter places a UCB over the probability of improvement.

\section{Batch BORE}
\label{sec:batch-bore}
This section proposes an extension of the BORE framework which allows for multiple queries to the objective function to be performed in parallel. Although many methods for batch BO have been previously proposed in the literature, %TODO: Cite a few
we here focus on approaching batch optimisation as an approximate Bayesian inference problem. %TODO: Cite previous optimisation-as-inference papers
Instead of having to derive complex heuristics to approximate the utility of a batch of query points, we can view points in a batch as samples from a posterior probability distribution which uses the acquisition function as a likelihood.

% \subsection{Optimisation as inference} %TODO: Remove this section and merge what's useful
% Let $f:\locDomain\to\R$ be a bounded continuous function over a set $\locDomain\subset\R^\locDim$, and let $p:\locDomain\to[0,\infty)$ define a probability density function over $\locDomain$. Since $f$ is bounded and $\int_\locDomain p(\location)\diff\location = 1$, the following integral is finite:
% \begin{equation}
%     \int_{\locDomain} p(\location) \exp{f(\location)}\diff\location \leq \exp\norm{f}_\infty < \infty\,,
%     \label{eq:f-normalisation}
% \end{equation}
% where $\norm{f}_\infty := \sup_{\location\in\locDomain} f(\location)$ denotes the supremum norm of $f$. As a result, we can define the following probability density function:
% \begin{equation}
%     p_f(\location) := \frac{p(\location)\exp f(\location)}{\int_\locDomain p(\location')\exp f(\location')\diff\location}
% \end{equation}
% It is easy to see that $p_f$ defines a valid probability density function. Assuming \autoref{eq:f-normalisation} is intractable, we can approximate $p_f$ by solving the following variational inference problem:
% \begin{equation}
%     q^* \in \argmin_{q\in\set{q}} \kl{q}{p_f}\,,
% \end{equation}
% where $\set{q}$ represents a set of variational distributions.

% This duality between optimisation and inference has been previously explored in the fields of optimal control \citep{Todorov2008}, reinforcement learning \citep{Fellows2019} and BO \citep{Gong2019}. With respect to the latter, however, we take a simpler but effective approach by not making use of risk measures in our framework.

\subsection{BORE batches via approximate inference}
Applying an optimisation-as-inference perspective to BORE, we can formulate a batch BO algorithm which does not require an explicit regression model for $\objective$. The classifier $\hat\classifier(\location) \approx p(\observation\leq \quantile|\location)$ naturally turns out as a likelihood function over query locations $\location\in\locDomain$. Since the search space $\locDomain$ is compact, we can assume a uniform prior distribution $p(\location) \propto 1$. Also note that the normalisation constant in this case is simply $\int_\locDomain p(\observation\leq\quantile|\location)p(\location) \diff\location = p(\observation\leq\quantile) = \percentile$. Our posterior distribution then becomes:
\begin{equation}
	\lowerpdf(\location) = p(\location|\observation\leq\quantile) = \frac{p(\observation\leq\quantile|\location)p(\location)}{p(\observation\leq\quantile)}\,.
	\label{eq:bore-posterior}
\end{equation}
Therefore, we formulate a batch version of BORE as an inference problem aiming for:
\begin{equation}
	\variational^* \in \argmin_{\variational \in \Pspace} \kl{\variational}{\lowerpdf},
\end{equation}
where $\kl{\variational}{\lowerpdf}$ denotes the Kullback-Leibler (KL) divergence between $\variational$ and $\lowerpdf$, and $\Pspace$ represents the space of probability distributions over $\locDomain$. Sampling from $\lowerpdf$ would allow us to obtain the points of interest in the search space, including the optimum $\location^*$ and other locations where $\observation\leq\quantile$. However, as the true $p(\observation\leq\quantile|\location)$ is unknown, we 
instead formulate a proxy inference problem with respect to a surrogate target distribution $\hat p_\iterIdx$ based on the classifier model. For BORE, we set $\hat p_\iterIdx(\location) \propto \hat\classifier_{\iterIdx-1}(\location)$, while for BORE++ the setting is $\hat p_\iterIdx(\location) \propto \classifier_{\iterIdx-1,\delta} (\location)$. In contrast to $\lowerpdf(\location) \propto p(\observation\leq\quantile|\location)$, the normalisation constant for the surrogate distributions is unknown, leading us to a proxy problem of minimising $\variational_\iterIdx \in \argmin_{\variational \in \Pspace} \kl{\variational}{\hat p_\iterIdx}$ at each iteration $\iterIdx \geq 1$.
This variational inference problem above can be efficiently solved via Stein variational gradient descent (SVGD) \citep{Liu2016}, described next.

{
	\subsection{Batch sampling via Stein variational gradient descent}
	\label{sec:svgd}
	In our implementation, we apply SVGD to approximately sample a batch $\set{B}_\iterIdx := \{\location_{\iterIdx,i}\}_{i=1}^\nSamples$ of $\nSamples \geq 1$ points from $\hat p_\iterIdx$. Other approximate inference algorithms could also be applied. One of the main advantages of SVGD, however, is that it encourages diversification in the batch, capturing the possible multimodality of $\hat p_\iterIdx$. Given the batch locations, observations can be collected in parallel and then added to the dataset to update the classifier model.
}

{
	SVGD  is an approximate inference algorithm which represents a variational distribution $\variational$ as a set of particles $\{\location^{i}\}_{i=1}^\nSamples$ \citep{Liu2016}. The particles are initialised as \iid{} samples from an arbitrary base distribution and then optimised via a sequence of smooth transformations towards the target distribution, which in our case corresponds to $\hat p_\iterIdx \propto \classifier_{\iterIdx-1, \delta}$. The SVGD steps are given by:
	\begin{align}
		\location^{i}_\iterIdx &\leftarrow \location^{i}_\iterIdx + \rate\sgrad_\iterIdx(\location^{i}_\iterIdx)~, \quad i \in \{1,\dots,\nSamples\},\\
		\sgrad_\iterIdx(\location) &:= \frac{1}{\nSamples} \sum_{j=1}^\nSamples k(\location^{j}_\iterIdx,\location)\nabla_{\location^{j}_\iterIdx}\log \classifier_{\iterIdx-1, \delta}(\location^{j}) + \nabla_{\location^{j}_\iterIdx} k(\location^{j}_\iterIdx, \location),
	\end{align}
	where $k:\locDomain\times\locDomain\to\R$ is a positive-definite kernel, and $\rate > 0$ is a small step size. Intuitively, the first term in the definition of $\vec{\zeta}_\iterIdx$ guides the particles to the modes of $\hat p_\iterIdx$, while the second term encourages diversification by repelling nearby particles. Theoretical convergence guarantees \citep{Liu2017, Korba2020} and practical extensions, such as second-order methods \citep{Detommaso2018, Liu2019} and derivative-free approaches \citep{Han2018}, have been proposed in the literature. Further details on SVGD can be found in \citet{Liu2016}.
}

%We initially sample a set of particles uniformly over the search space $\locDomain$ and run the algorithm's steps as:
%\begin{align}
%	\location_i^{t+1} &= \location_i^{t} + \factor_\iterIdx \vec{\zeta}_\iterIdx(\location_i^t)~,\\
%	\vec{\zeta}_\iterIdx(\location) &:= \frac{1}{\nSamples} \sum_{j=1}^\nSamples k(\location_j^{t},\location)\nabla_{\location_j^{t}}\log \classifier(\location_j^{t}) + \nabla_{\location_j^{t}} k(\location_j^{t}, \location)\,.
%\end{align}

% In the original BORE algorithm, we select a point at a time by maximising the classifier probability output $\classifier_\iterIdx(\location) \approx p(\observation \leq \quantile_\iterIdx | \location, \dataset_\iterIdx)$. The target term $p(\observation \leq \quantile | \location, \dataset_\iterIdx)$ plays the role of a likelihood in a Bayesian inference analogy:
% \begin{equation}
%     \lowerpdf(\location|\dataset_\iterIdx) := p(\location | \observation \leq \quantile, \dataset_\iterIdx) = \frac{p(\observation\leq \quantile|\location, \dataset_\iterIdx)p(\location|\dataset_\iterIdx)}{p(\observation \leq \quantile|\dataset_\iterIdx)} \,.
% \end{equation}
% If we consider the search space $\locDomain$ as having a uniform prior $p(\location)$, with BORE we are selecting the point of maximum probability under $\lowerpdf(\location)$. An alternative approach, however, is to select a batch of points sampled from $\lowerpdf(\location)$. This inference can be performed in a non-parametric setting via SVGD.

\subsection{Regret bound for Batch BORE++ {with PLS classifiers}}
%The batch version of BORE++ can be interpreted as trying to minimise the KL divergence between a variational distribution $\variational_\iterIdx(\location)$ and the target $\lowerpdf(\location) := p(\location|\observation\leq\quantile)$, which is sequentially approximated by $\hat p_\iterIdx(\location) \propto \classifier_{\iterIdx,\delta}(\location)$ over iterations $\iterIdx\in\{1,\dots,\nIterations\}$. Under these observations, 
We follow the derivations in \citet{Oliveira2021} to derive a distributional regret bound for batch BORE++ with respect to its target sampling distribution $\lowerpdf$, which is presented in the following result.

\begin{theorem}
	\label{thr:batch-pls-regret}
	Under the same assumptions in \autoref{thr:pls-ucb}, running batch BORE++ with $\classifier_{\iterIdx,\delta}$ set as in \autoref{eq:pls-ucb-classifier}, we obtain a bound on the instantaneous distributional regret:
	\begin{equation}
		\eregret_\iterIdx := \expectation_{\location\sim \hat p_\iterIdx}[\objective(\location)] - \expectation_{\location\sim \lowerpdf}[\objective(\location)] \leq 2\Lipschitz_\obsNoise\Lipschitz_\classifier\beta_{\iterIdx-1}(\delta)\expectation_{q_\iterIdx}[\sigma_{\iterIdx-1}]\,, \quad \iterIdx \geq 1\,,
	\end{equation}
	where $\Lipschitz_\classifier := \max_{\location\in\locDomain} \frac{1}{\classifier(\location)}$, and on the cumulative distributional regret at $\nIterations\geq 1$:
	\begin{equation}
		\eRegret_\nIterations := \sum_{\iterIdx=1}^\nIterations \eregret_\iterIdx \leq 4\Lipschitz_\obsNoise\Lipschitz_\classifier\beta_\nIterations(\delta) \sqrt{(\nIterations+2)\mig_\nIterations} \in \set{O}(\sqrt{\nIterations}(\bound\sqrt{\mig_\nIterations}+\sqrt{\mig_\nIterations\mig_{\nSamples\nIterations}}))
	\end{equation}
	both of which hold with probability at least $1-\delta$.
\end{theorem}

As in the case of non-batch BORE++, the distributional regret bounds for the batch algorithm also grow sub-linearly for most popular kernels, leading to an asymptotically vanishing simple regret. Although different, to compare the distributional regret of batch BORE++ with the non-distributional regret bounds for BORE++, we may consider a case where $\quantile$ is set to the function minimum $\quantile := \objective(\location^*) = \min_{\location\in\locDomain} \objective(\location)$ and the observation noise is small. In this case, the batch sampling distribution would converge to a Dirac at the optimum, so that $\expectation_{\location\sim\lowerpdf}[\objective(\location)] \approx \objective(\location^*)$. Compared to the regret of non-batch BORE++ (\autoref{thr:pls-regret}) after collecting an equivalent number of observations $\nIterations' := \nSamples\nIterations$, the expected regret of the batch version of BORE++ {after $\nIterations$ iterations} is then lower by a factor of $\mig_{\nIterations}/\mig_{\nSamples\nIterations}$, noting that $\mig_\nIterations\leq \mig_{\nIterations'} = \mig_{\nSamples\nIterations}$. {Therefore, batch BORE++ should be able to achieve lower regret per runtime than sequential BORE++ with an equivalent number of observations.}

% \section{THEORETICAL ANALYSIS}
% This section presents a theoretical analysis of the proposed framework. We start by presenting a few results concerning the original BORE framework and then extend it to the batch setting and the UCB formulation in BORE++. The following result provides a basis for our analysis, as it shows that BORE's target predictive distribution leads us to the global optimum of the objective function.

% A PLS GP classifier models a response function $\gpFunction\sim\gp(0,k)$ which is squashed via a link function $\linkf:\R\to[0,1]$, so that $p(\clabel=1|\location,\gpFunction) = \linkf(\gpFunction(\location))$. Parameters of the link function and the hyper-parameters of the GP model can both be learned via leave-one-out cross validation (see details in \citet[Sec. 6.5]{Rasmussen2006}). However, here we assume these parameters have been pre-determined so that $\objective \in \Hspace$.\footnote{For example, we may assume that $\linkf = \cdf_\obsNoise$. Then, for $\gpFunction\in\Hspace$, $p(\clabel=1|\location,\objective) = \cdf_\obsNoise(\quantile-\objective(\location)) = \linkf(\gpFunction(\location)), \forall\location\in\locDomain \implies \quantile - \objective(\location) = \gpFunction(\location), \forall\location\in\locDomain \implies \objective\in\Hspace$.}

\section{Related work}
% TODO:
% - Level-set methods (A. Krause)
% - Relative least-squares importance fitting (Sugiyama)
Since their proposal by \citet{Schonlau1998}, batch Bayesian optimisation methods have appeared in various forms in the literature. Many methods are based on heuristics derived from estimates given by a Gaussian process regression model \citep{Gonzalez2016batch, Wang2017batch, Azimi2010}. Others are based on Monte Carlo estimates of multi-query acquisition functions \citep{Snoek2012, Wilson2018}, optimising points over GP posterior samples \citep{Kandasamy2018}, solving local optimisation problems \citep{Eriksson2019}, or optimising over ensembles of acquisition functions \citep{Zhang2022}. Despite that, the prevalent approaches to batch BO are still based on a GP regression model, which require prior knowledge about the objective function and do not scale to high-dimensional problems. We instead take a different approach by viewing BO as a density-ratio estimation problem following the BORE framework by \citet{tiao2021bore}. For batch design, we take an optimisation-as-inference approach \citep{Todorov2008, Fellows2019} by applying Stein variational gradient descent, a non-parametric approximate inference method \citep{Liu2016}, which has been recently combined with GP-based BO \citep{Gong2019, Oliveira2019aabi}. Our theoretical results, however, are agnostic to the choice of inference algorithm. In contrast to traditional batch BO methods, the inference approach does not require solving inter-dependent optimisation problems for each batch point, as in heuristic-based approaches \citep{Azimi2010, Gonzalez2016batch, Wang2017batch}, Monte Carlo integration over the GP posterior \citep{Snoek2012, Wilson2018}, nor sampling from it \citep{Kandasamy2018}. SVGD allows batch selection to be solved in a vectorised way, which can take advantage of hardware accelerators, such as GPUs.

\section{Experiments}

This section presents experiments assessing the theoretical results and demonstrating the practical performance of batch BORE on a series of global optimisation benchmarks. We compared our methods against GP-based BO baselines in both experiments sets. Additional experimental results, including the sequential setting (Appendix E), a description of the experiments setup (Appendix E), and further discussions on theoretical aspects can be found in the supplementary material.\footnote{Code will be made available at \url{https://github.com/rafaol/batch-bore-with-guarantees}}

%\begin{figure}
%	\centering
%	\subfloat[Regret]{\includegraphics[width=0.31\textwidth]{figures/theory-regret-1D.png}}
%	
%	\subfloat[Objective function example]{\includegraphics[width=0.31\textwidth]{figures/theory-example.png}}
%	\caption[Theory assessment]{Theory assessment experiment. The plots show the averaged regret per iteration. Results were averaged over 10 trials, and the shaded area indicates the 95\% confidence interval.\footnotemark}
%	\label{fig:toy}
%\end{figure}

%\begin{minipage}{0.6\textwidth}
\paragraph{Theory assessment.} We first present simulated experiments assessing the theoretical results in practice, testing BORE and BORE++ in the PLS setting. As a baseline, we compare both methods against GP-UCB. This experiment was run by generating a random base classifier in the RKHS $\Hspace_k$ and then a corresponding objective function via the inverse noise CDF $\cdf_\obsNoise^{-1}$. The search space was set as a uniformly-sampled finite subset of the unit interval $\locDomain := [0,1] \subset\R$. We applied the theory-backed settings for BORE++ (\autoref{sec:bore}) and GP-UCB \citep{Durand2018}, while BORE employed the optimal PLS classifier (\autoref{eq:pls-estimator}).

\begin{wrapfigure}{r}{0.3\textwidth}
	\centering
	\includegraphics[width=0.3\textwidth]{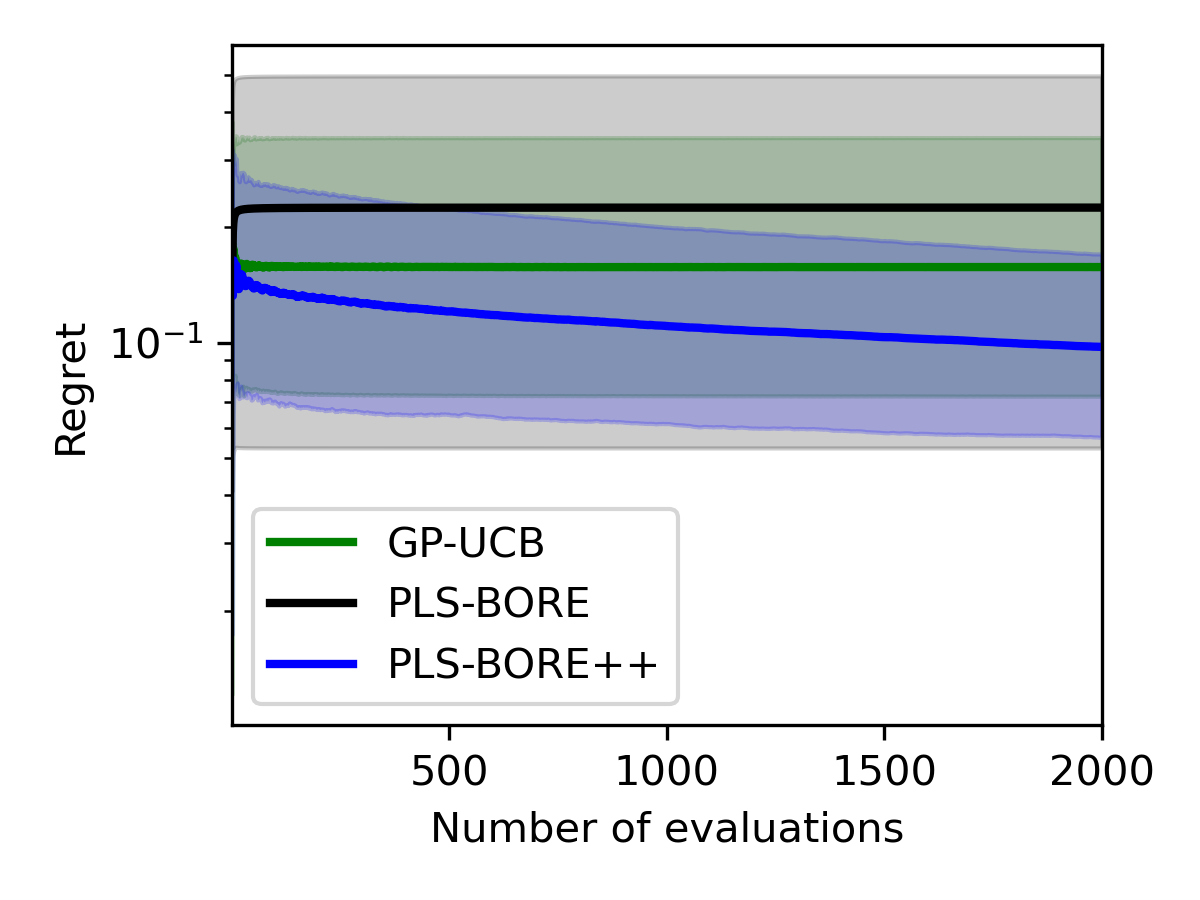}
	\captionof{figure}[Theory assessment]{Regret in theory assessment experiment. Results were averaged over 10 trials, and the shaded area indicates the 95\% confidence interval.\footnotemark}
	\label{fig:toy}
\end{wrapfigure}

As the results in \autoref{fig:toy} show, BORE using an optimal PLS classifier simply gets stuck at a its initial point, resulting in constant regret. BORE++, however, is able to progress in the optimisation problem towards the global optimum, outperforming the GP-UCB baseline.
%\end{minipage}
%\hspace{2mm}
%\begin{minipage}{0.35\textwidth}
%	\centering
%	\includegraphics[width=0.8\textwidth]{figures/theory-regret-1D.png}
%	\captionof{figure}[Theory assessment]{Regret in theory assessment experiment. Results were averaged over 10 trials, and the shaded area indicates the 95\% confidence interval.\footnotemark}
%	\label{fig:toy}
%\end{minipage}

\begin{figure}
	\centering
	\subfloat{\includegraphics[width=0.31\textwidth]{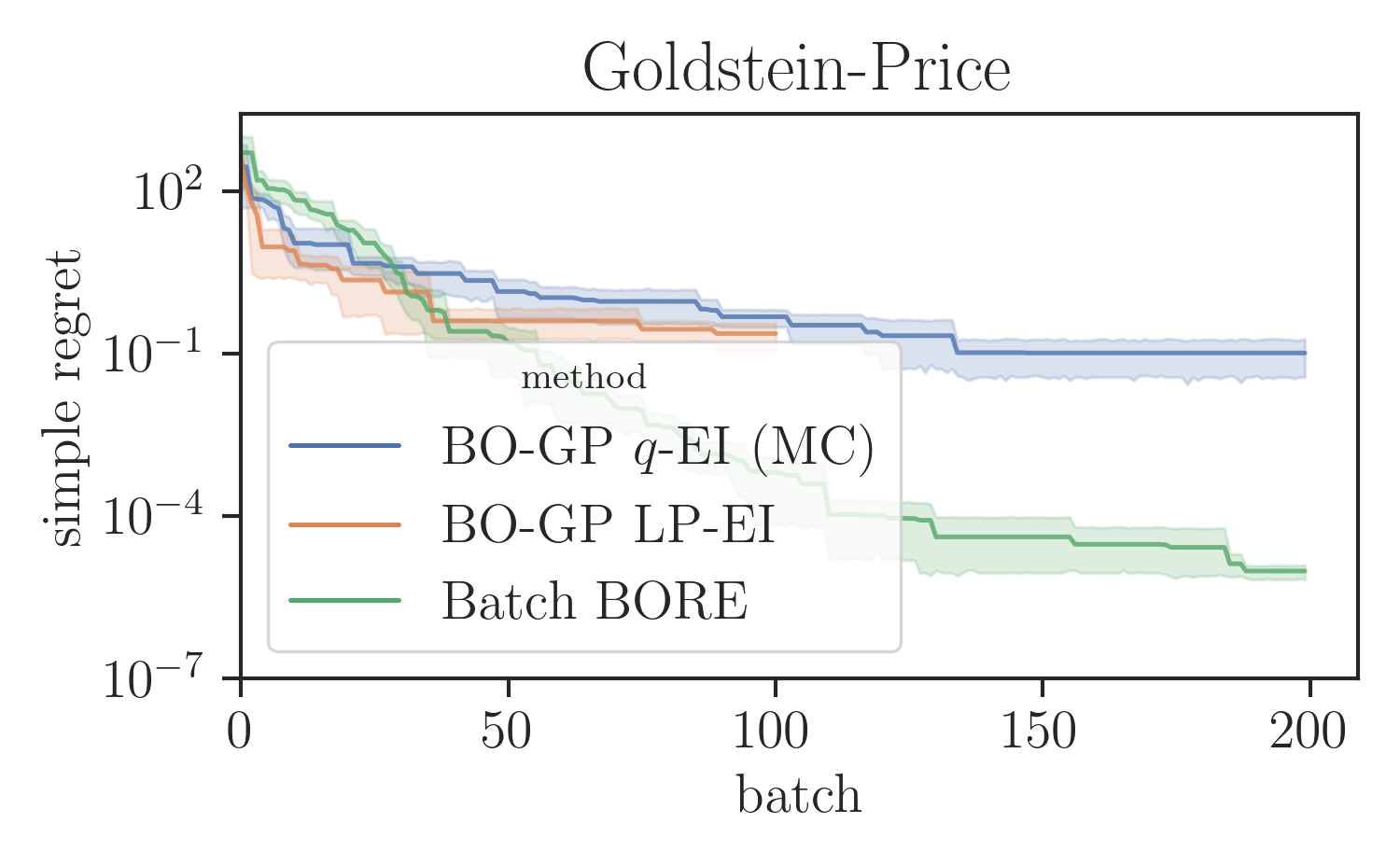}} % bore-logit-relu-10
	\subfloat{\includegraphics[width=0.31\textwidth]{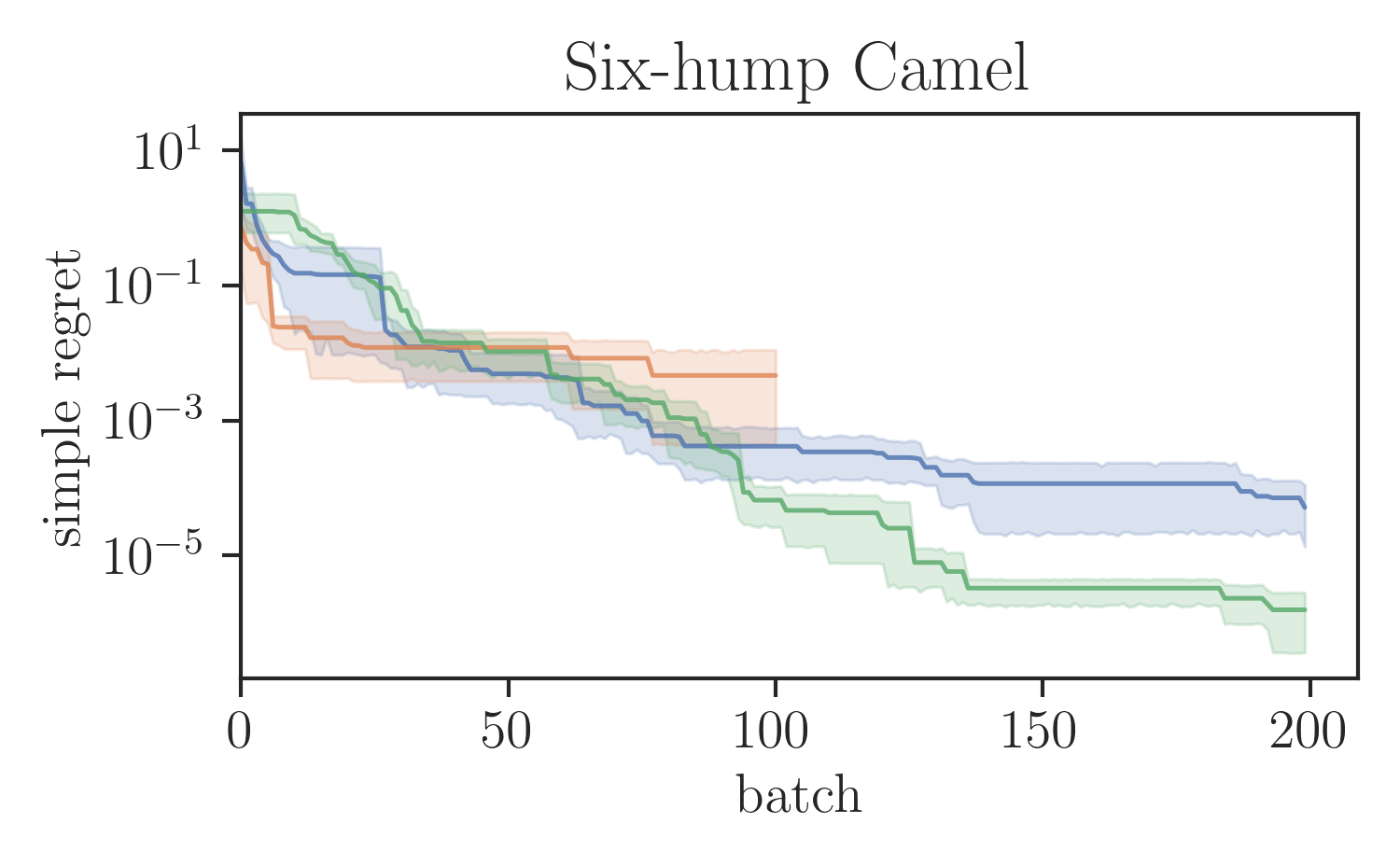}} % bore-logit-steps-200-10
	\subfloat{\includegraphics[width=0.31\textwidth]{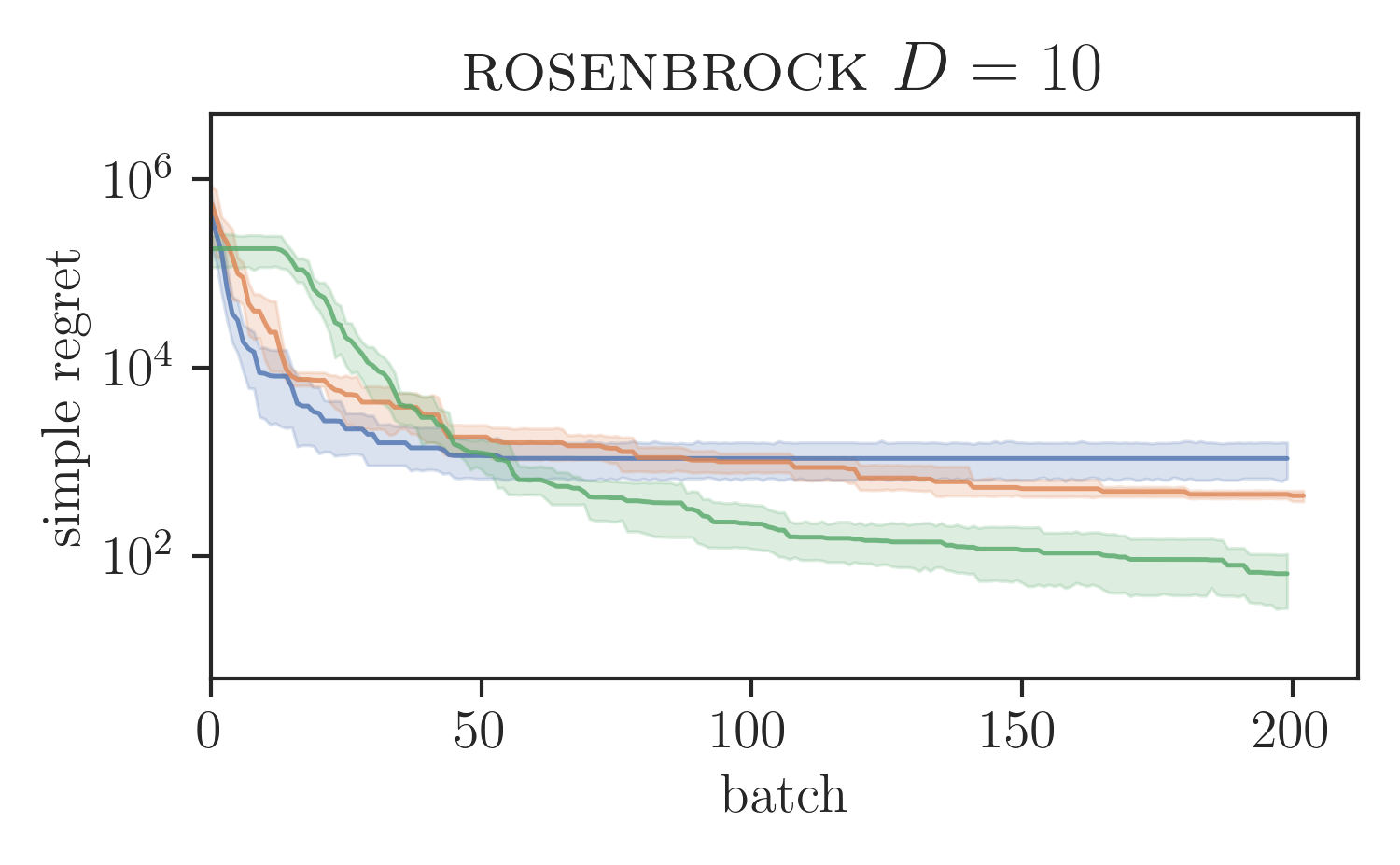}} % bore-logit-relu-10
	
	\subfloat{\includegraphics[width=0.3\textwidth]{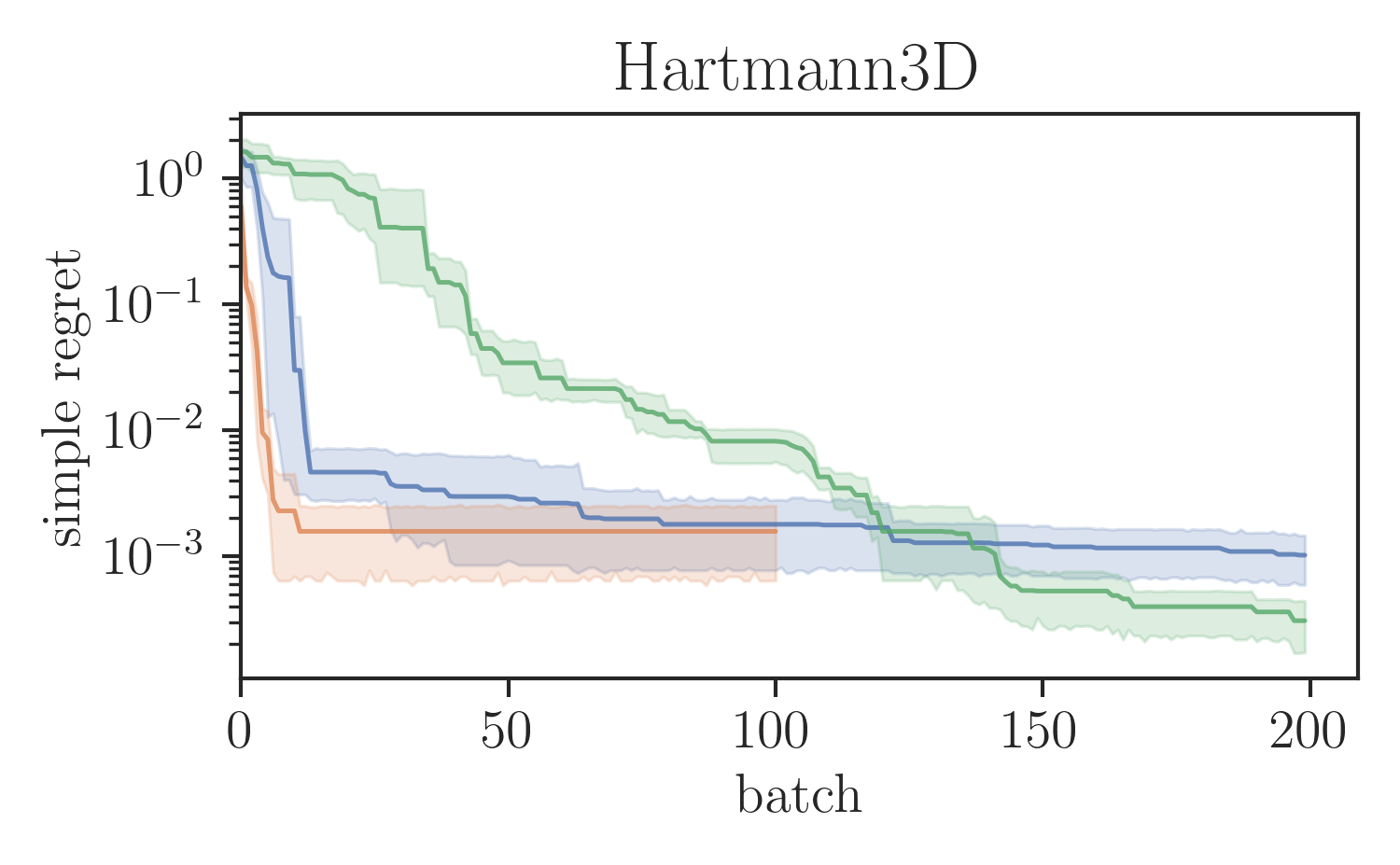}} % bore-logit-max-iter-500-10
	\subfloat{\includegraphics[width=0.3\textwidth]{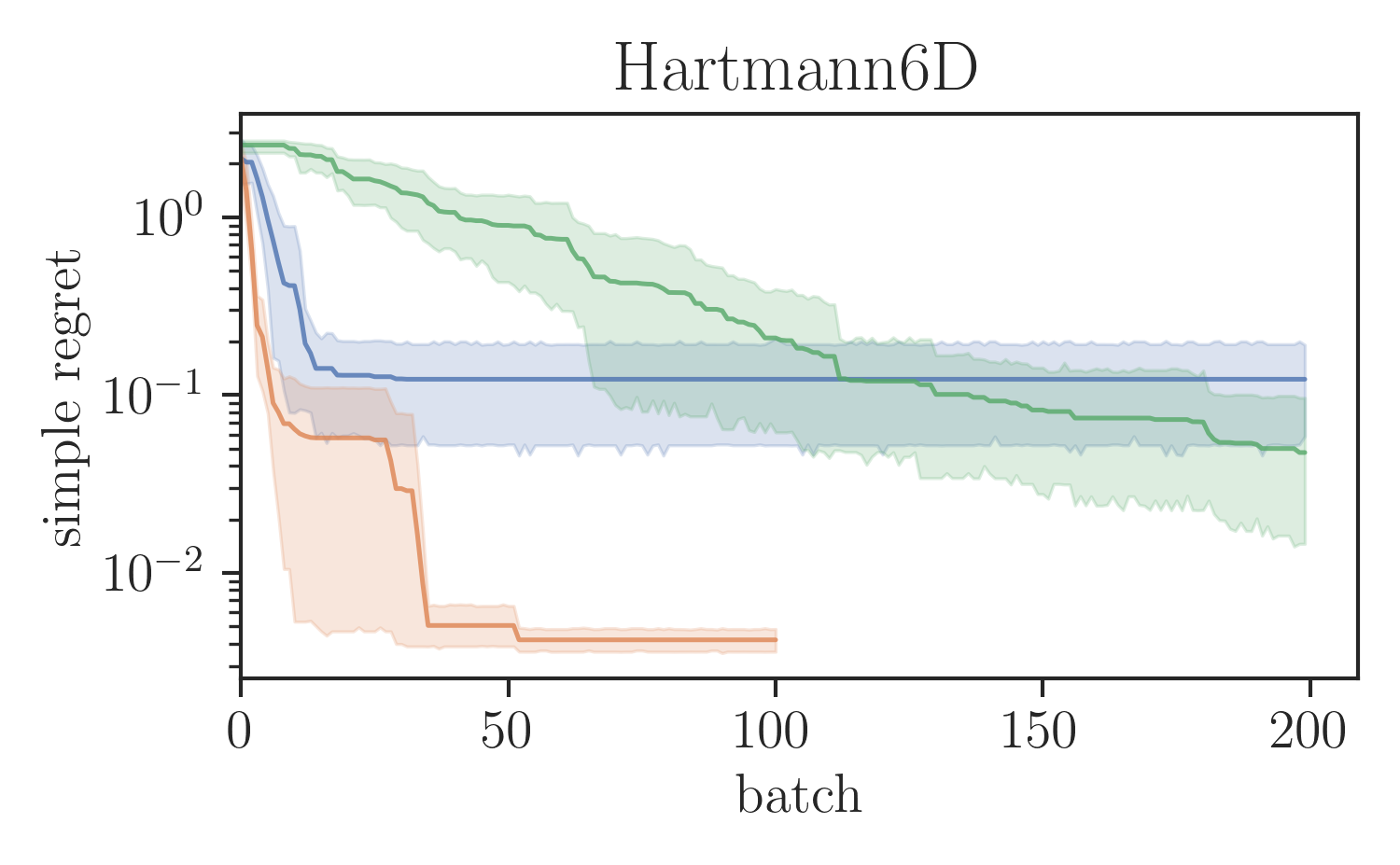}} % bore-logit-relu-10
	\subfloat{\includegraphics[width=0.31\textwidth]{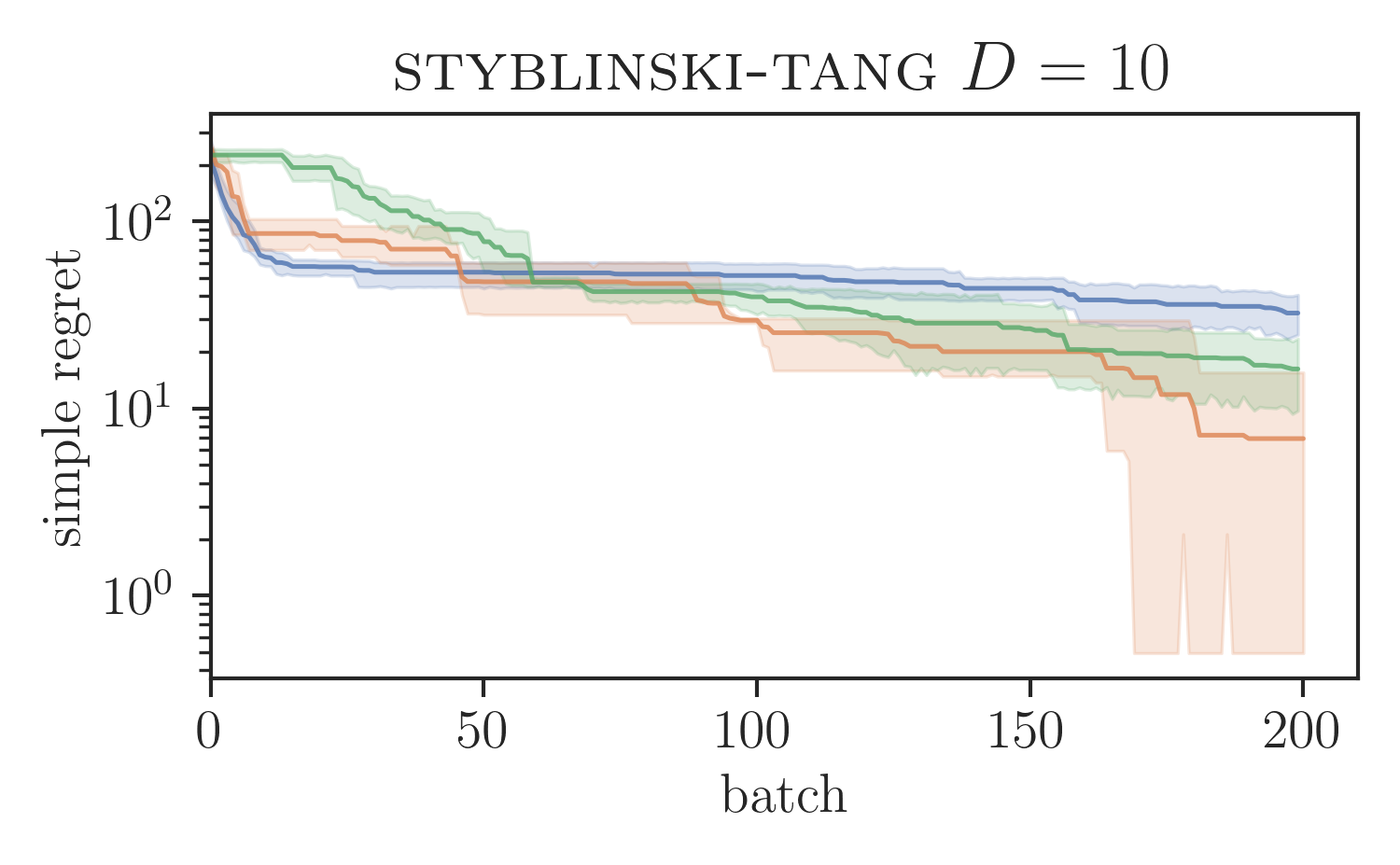}} % bore-logit-relu-10
	\caption{Performance on synthetic benchmarks. Plots show the simple regret, i.e., $\min_{\iterIdx\leq\nIterations}\regret_\iterIdx$, per iteration. Results were averaged over 5 trials, and shaded areas indicate the 95\% confidence interval.}
	\label{fig:qs-bore-basic}
\end{figure}

\paragraph{Global optimisation benchmarks.} %TODO: Add details on problem dimension.
We evaluated the proposed SVGD-based batch BORE method in a series of test functions for global optimisation comparing it against other BO baselines. In particular, for our comparisons, we ran the locally penalised EI (LP-EI) method \citep{Gonzalez2016batch} and the Monte Carlo based $q$-EI method \citep{Snoek2012}, which are both based on the EI algorithm, like BORE. Results are presented in Figure \ref{fig:qs-bore-basic}. All methods ran for $\nIterations := 200$ iterations and used of batch size of 10 evaluations per iteration. Additional experimental details are deferred to the supplementary material.\footnotetext{Linear interpolation is applied to obtain the plotted confidence intervals when the number of trials is small.}

As \autoref{fig:qs-bore-basic} shows, batch BORE is able to outperform its baselines on most of the global optimisation benchmarks. We also note that, in some case, due to its complexity the LP-EI method becomes computationally infeasible after 100 iterations, having to be aborted halfway through the optimisation. Batch BORE, however, is able to maintain steady performance throughout its runs.

\paragraph{Real-data benchmarks.}
Lastly, we compared the sequential version of BORE++ against BORE and other baselines, including traditional BO methods, such as GP-UCB and GP-EI \citep{Shahriari2016}, the Tree-structured Parzen Estimator (TPE) \citep{bergstra2011algorithms}, and random search, on real-data benchmarks. In particular, we assessed the algorithms on some of the same benchmarks present in the original BORE paper \citep{tiao2021bore}. 

\begin{figure}
	\centering
	\subfloat[Hyper-parameter tuning on Parkinson's dataset]{\includegraphics[width=0.5\textwidth]{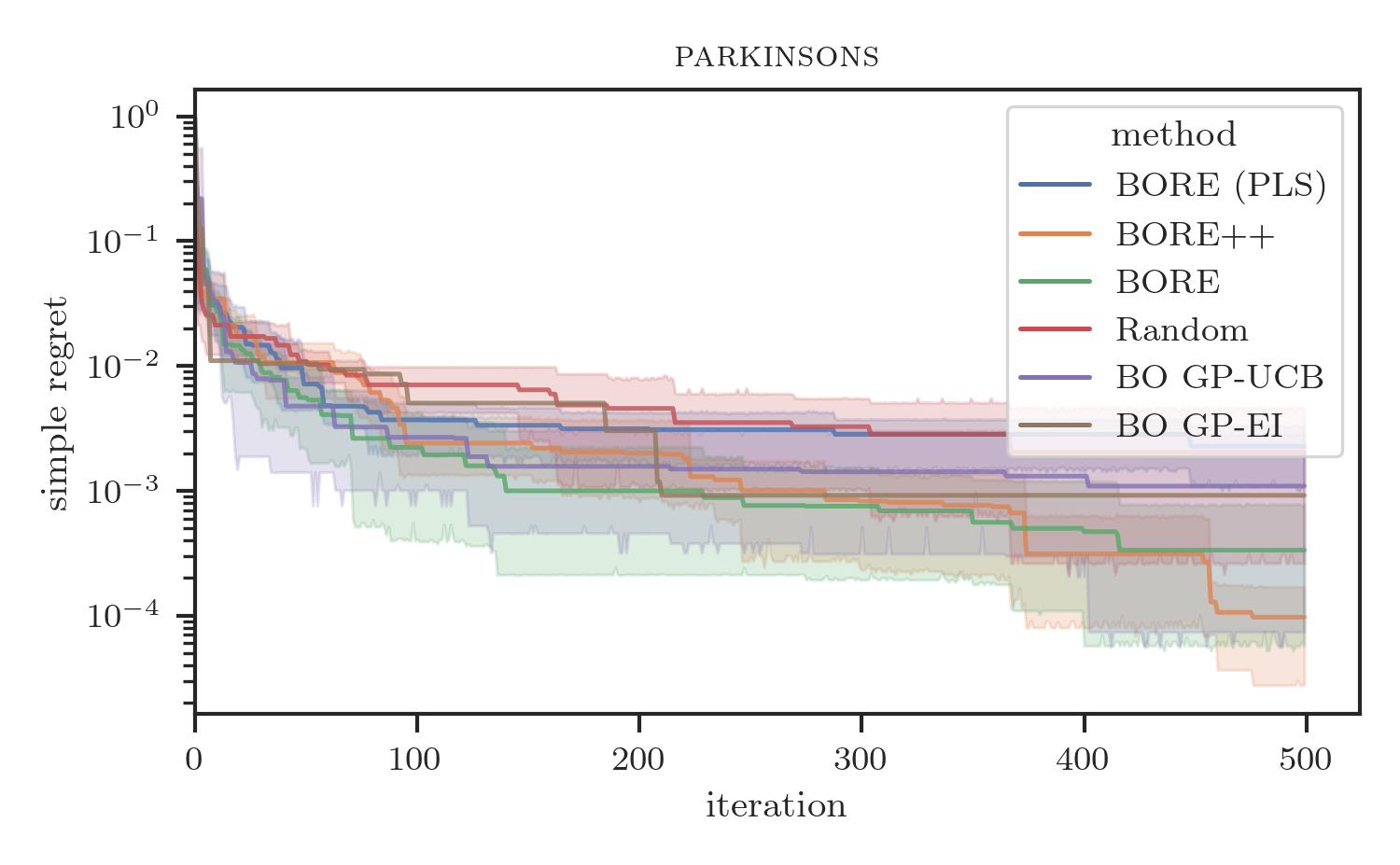}\label{fig:parkinsons}}%
	\subfloat[Hyper-parameter tuning on CT slice dataset]{\includegraphics[width=0.5\textwidth]{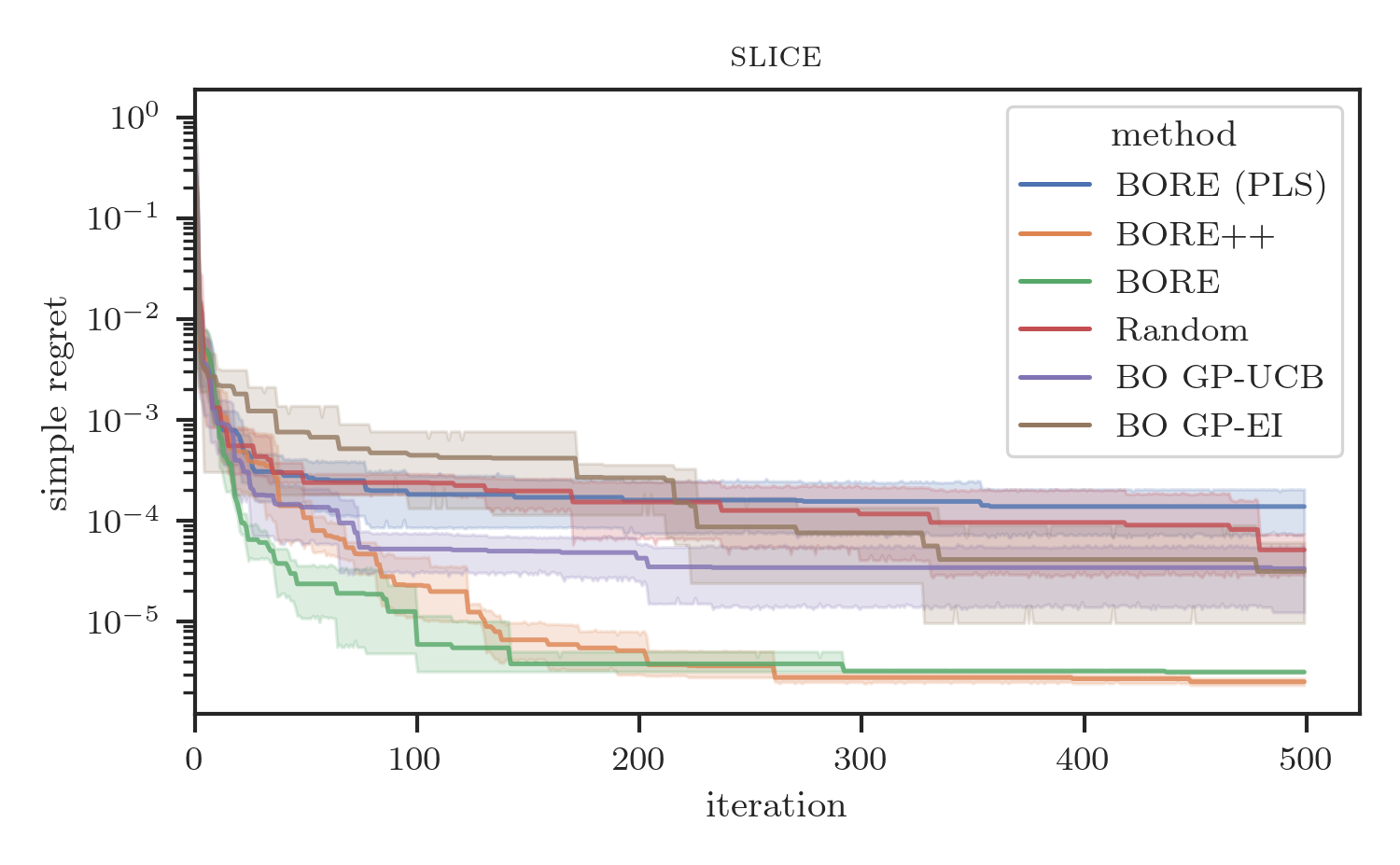}\label{fig:slice}}
	
	\subfloat[Racing line optimisation]{\includegraphics[width=0.5\textwidth]{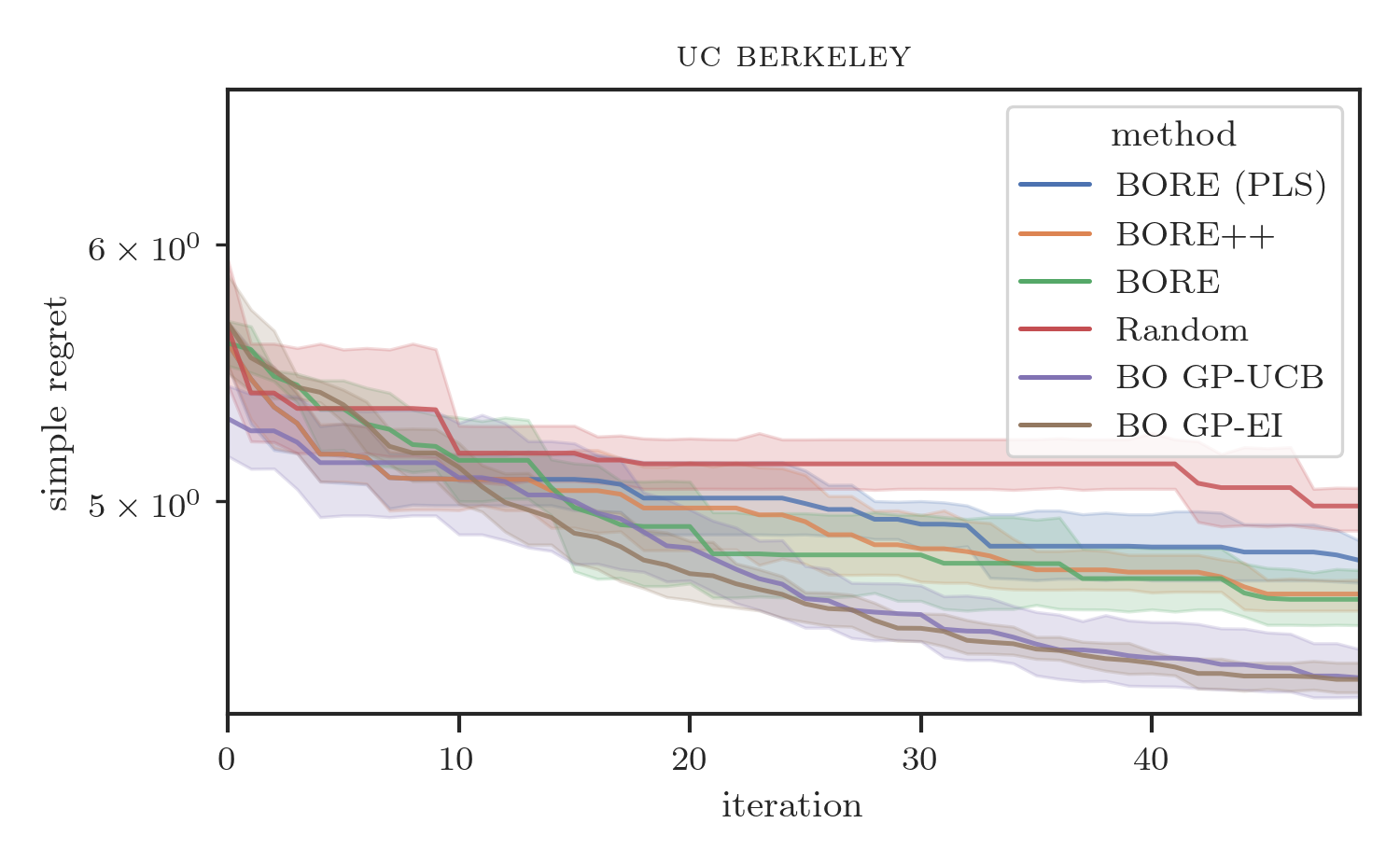}\label{fig:racing}}
	\subfloat[Neural architecture search on MNIST data]{\includegraphics[width=0.5\textwidth]{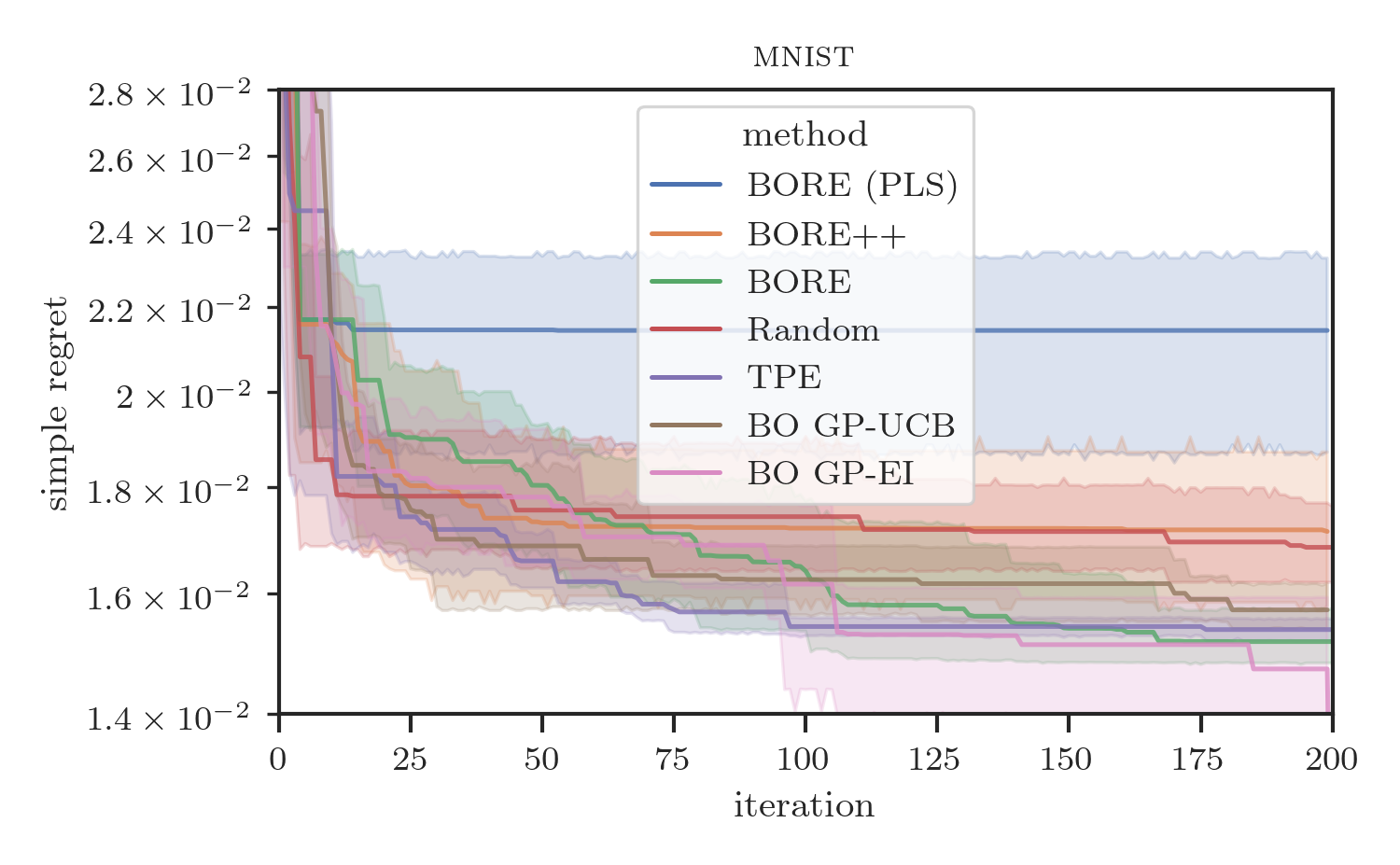}\label{fig:mnist}}
	\caption{Experimental results on real-data benchmarks. The plots show each algorithm's simple regret averaged across multiple runs. The shaded areas correspond to the 95\% confidence intervals.}
	\label{fig:real-data}
\end{figure}

Results are presented in \autoref{fig:real-data}. As the plots show, BORE++ presents significantly better performance than BORE in the probabilistic least-squares (PLS) setting (i.e., $\beta_\iterIdx := 0$), as the theoretical results suggested. In fact, it is possible to note that BORE (PLS) performs comparably to (or at times worse than) random search, indicating that the optimal least-squares classifier by itself is unable to properly capture the epistemic uncertainty. By using a neural network classifier trained via gradient descent and a different loss function (cross-entropy), the original BORE is still able to achieve top performance in most benchmarks. Both BORE versions are only surpassed by traditional GP-based BO on the racing line optimisation problem, as observed in \citet{tiao2021bore}, due to the inherent smoothness the problem, and in the final iterations of the neural architecture search problem by GP-EI. Interestingly, even though restricted to the kernel-based PLS setting, we observe that BORE++ is able to surpass the original BORE in the neural network hyper-parameter tuning problems (\textsc{slice} and \textsc{parkinsons}), while maintaining similar performance in other tasks. These results confirm that improved uncertainty estimates can lead to practical performance gains.

%\subsection{Racing line optimisation}
%
%\begin{figure*}
% \centering
% \includegraphics[width=\textwidth]{figures/qs-bore-racing.png}
% \caption{Caption}
% \label{fig:qs-bore-racing}
%\end{figure*}

\section{Conclusion}
This paper presented an extension of the BORE framework to the batch setting alongside the theoretical analysis of the proposed extension and an improvement over the original BORE. Theoretical results in terms of regret bounds and experiments show that BORE methods are able to maintain performance guarantees while outperforming traditional BO baselines. The main purpose of this work, however, was to establish the theoretical foundations for the analysis and derivation of new algorithmic frameworks for Bayesian optimisation via density-ratio estimation, equipping BO with new tools based on probabilistic classification, instead of regression models.

As future work, we plan to investigate the theoretical properties of BORE under different loss functions and analyse other batch design strategies, e.g., sampling methods for combinatorial domains. The theoretical contributions of this work can also be extended to other versions of BORE, such as its recent multi-objective version \citep{DeAth2022mbore}, and provide insights into other likelihood-free BO methods \citep{Song2022lfbo}. Finally, we consider integrating BORE++ with other probabilistic classification models equipped with predictive uncertainty estimates, such as neural network ensembles \citep{Lakshminarayanan2017}, random forests \citep{Balog2015}, and Bayesian generalised linear models, which should lead to improvements in scalability and additional performance gains.

\begin{ack}
Rafael Oliveira was supported by the Medical Research Future Fund Applied Artificial Intelligence in Health Care grant (MRFAI000097) by the Australian Department of Health and Aged Care.
\end{ack}

	%TODO: Comment on:
	% - we did not analyse the effect of the initial design on the regret
	% - we only considered PLS classifiers in the theoretical analysis
	
	\bibliographystyle{unsrtnat}
	\bibliography{references,extra}

	\section*{Appendix}
\appendix

\counterwithin{equation}{section}
\counterwithin{lemma}{section}
\counterwithin{proposition}{section}
\counterwithin{assumption}{section}
\counterwithin{theorem}{section}

\counterwithin{figure}{section}
\counterwithin{table}{section}

This appendix complements the main paper with proofs, experiment details and additional experiments and discussions.  \autoref{sec:proofs} presents full proofs for the main theoretical results in the paper. In \autoref{sec:bore-thompson}, we discuss an approach to derive alternative regret bounds for BORE under a Thompson sampling perspective. We discuss the theoretical analysis of BORE with its non-constant approximation for the observations quantile $\quantile$ in \autoref{sec:non-constant}. In \autoref{sec:exp-details}, we present further details on the experiments setup. Finally, \autoref{sec:extra-exp} presents an additional experiment assessing dimensionality effects.

\appendix

\section{Proofs}
\label{sec:proofs}
This section presents proofs for the main theoretical results in the paper. We start with a few auxiliary results from the GP-UCB literature \citep{Srinivas2010, Chowdhury2017}, following up with the proofs for the main theorems.

\subsection{Auxiliary results}
\label{sec:proofs-aux}

\begin{lemma}[{\citet[Lemma 5.3]{Srinivas2010}}]
	\label{thr:ig}
	The information gain for a sequence of $\nObs \geq 1$ observations $\{\location_i, \clabel_i\}_{i=1}^\nObs$, where $\clabel_i = f(\location_i) + \gpnoise_i$, $\gpnoise_i\sim\normal(0,\regFactor)$, can be expressed in terms of the predictive variances. Namely, if $f\sim\gp(\gpMeanFunction,k)$, then the information gain provided by the observations is such that:
	\begin{equation}
		\mutinfo{\vec\clabel_\nObs,\vec{f}_\nObs|\locDomain_\nObs} = \frac{1}{2}\sum_{i=1}^\nObs \log(1+\regFactor^{-1}\sigma^2_{i-1}(\location_i))~,
	\end{equation}
	where $\vec{f}_\nObs := [f(\location_i)]_{i=1}^\nObs$ and $\locDomain_\nObs := \{\location_i\}_{i=1}^\nObs \subset \locDomain$.
\end{lemma}

\begin{lemma}[{\citet[Lemma 4]{Chowdhury2017}}]
	\label{thr:variance-sum-bound}
	Following the setting of \autoref{thr:ig}, the sum of predictive standard deviations at a sequence of $\nObs$ points is bounded in terms of the maximum information gain:
	\begin{equation}
		\sum_{i=1}^\nObs \sigma_{i-1}(\location_i) \leq \sqrt{4(\nObs+2)\mig_\nObs}\,.
	\end{equation}
\end{lemma}

\begin{lemma}
	\label{thr:batch-variance}
	Let $\set A \subset\locDomain$ be a finite set of points where a function $f\sim\gp(\gpMeanFunction,k)$ was evaluated, so that the GP posterior covariance function and the corresponding variance are given by:
	\begin{align}
		k_{\set A}(\location,\location') &:= k(\location,\location') - k(\location, \set A)^\transpose(\mat k (\set A) + \eta\eye)^{-1} k(\set A, \location')\\
		\sigma_{\set A}^2(\location) &:= k_{\set A}(\location,\location)\,,\quad\location, \location' \in\locDomain,
	\end{align}
	where $k(\location, \set A) := [k(\location,\vec a)]_{\vec a\in\set A}$ and $\mat k (\set A) := [k(\vec a, \vec a')]_{\vec a, \vec a' \in \set A}$.
	Then, for any given set $\set B \supset \set A$ of evaluations of $f$, we have:
	\begin{equation}
		\sigma_{\set B}^2(\location) \leq \sigma_{\set A}^2(\location), \quad \forall \location \in \locDomain\,.
	\end{equation}
\end{lemma}
\begin{proof}
	The result follows by observing that the GP posterior given observations at $\set A$ is a prior for the GP with the new observations at the complement $\set{C} := \set B \setminus \set A$. Then we obtain, for all $\location\in\locDomain$:
	\begin{equation}
		\begin{split}
			\sigma_{\set B}^2(\location) &:= k(\location,\location) - k(\location,\set B)(\mat k(\set B) + \eta\eye)^{-1}k(\set B,\location)\\
			&= \sigma_{\set A}^2(\location) - k_{\set A}(\location, \set C)(\mat k_{\set A}(\set C) + \eta \eye)^{-1} k_{\set A}(\set C, \location)\\
			&\leq \sigma_{\set A}^2(\location)\,,
		\end{split}
	\end{equation}
	since $k_{\set A}(\location, \set C)(\mat k_{\set A}(\set C) + \eta \eye)^{-1} k_{\set A}(\set C, \location)$ is non-negative. %TODO: We can expand the proof via an induction argument which adds point by point to the GP.
\end{proof}

\subsection{Proof of Theorem 1}
%TODO: Restate Durand et al.'s theorem here or somewhere else in the appendix.
\begin{proof}[Proof of Theorem 1]
	The proof follows by a simple application of \citet[Thm. 1]{Durand2018} on GP-UCB to our settings, as $\classifier\in\Hspace$ and the stochastic process defining the query locations $\location_\iterIdx$ and observation noise $\gpnoise_\iterIdx := \clabel_\iterIdx - \classifier(\location_\iterIdx)$ satisfies their assumptions of sub-Gaussianity. In particular, $\gpnoise_\iterIdx$ is $\sigma_\gpnoise$-sub-Gaussian with $\sigma_\gpnoise \leq 1$, since $|\gpnoise_\iterIdx| \leq 1$, for all $\iterIdx\geq 1$ \citep{Boucheron2013}.
\end{proof}

\subsection{Proof of Theorem 2}
To prove Theorem 2, we will follow the procedure of GP-UCB proofs \citep{Srinivas2010, Chowdhury2017} by bounding the approximation error $|\classifier(\location) - \hat\classifier_\iterIdx(\location)|$ via a confidence bound (Theorem 1) and then applying it to the instant regret. From the instant regret to the cumulative regret, the bounds are extended by means of the maximum information gain $\mig_\nIterations$ introduced in the main text. One of the differences with our proof, however, is that BORE with a PLS classifier is not following the optimal UCB policy, but instead a pure-exploitation approach by following the maximum of the mean estimator $\hat\classifier_\iterIdx$, which does not account for uncertainty.

\begin{proof}[Proof of Theorem 2]
	Recalling the classifier-based bound in Section 4 and that for any $\quantile\in\R$ the result in Lemma 1 holds, we have:
	\begin{equation}
		\begin{split}
			\regret_\iterIdx &= f(\location_\iterIdx) - f(\location^*)\\
			&\leq\Lipschitz_\obsNoise(\classifier(\location^*) - \classifier(\location_\iterIdx))
			\label{eq:classifier-regret}
		\end{split}
	\end{equation}
	According to Theorem 1, working with the confidence bounds on $\classifier(\location)$, we then have that the instant regret is bounded with probability at least $1-\delta$ by:
	\begin{equation}
		\begin{split}
			\forall\iterIdx\geq 1,\quad \regret_\iterIdx &\leq \Lipschitz_\obsNoise (\hat\classifier_{\iterIdx-1}(\location^*) + \beta_{\iterIdx-1}(\delta)\sigma_{\iterIdx-1}(\location^*) - \classifier(\location_\iterIdx))\\
			&\leq \Lipschitz_\obsNoise (\hat\classifier_{\iterIdx-1}(\location^*) + \beta_{\iterIdx-1}(\delta)\sigma_{\iterIdx-1}(\location^*) - \hat\classifier_{\iterIdx-1}(\location_\iterIdx) + \beta_{\iterIdx-1}(\delta)\sigma_{\iterIdx-1}(\location_\iterIdx))\\
			&\leq \Lipschitz_\obsNoise\beta_{\iterIdx-1}(\delta)(\sigma_{\iterIdx-1}(\location^*) + \sigma_{\iterIdx-1}(\location_\iterIdx)),
		\end{split}
	\end{equation}
	since $\hat\classifier_{\iterIdx-1}(\location^*) \leq \max_{\location\in\locDomain}\hat\classifier_{\iterIdx-1}(\location) = \hat\classifier_{\iterIdx-1}(\location_\iterIdx)$.
	Now we can apply \autoref{thr:variance-sum-bound}, yielding with probability at least $1-\delta$:
	\begin{equation}
		\begin{split}
			\Regret_\nIterations := \sum_{\iterIdx=1}^\nIterations \regret_\iterIdx &\leq \Lipschitz_\obsNoise \beta_\nIterations(\delta)\sum_{\iterIdx=1}^\nIterations (\sigma_{\iterIdx-1}(\location_\iterIdx) + \sigma_{\iterIdx-1}(\location^*))\\
			&\leq \Lipschitz_\obsNoise \beta_\nIterations(\delta) \left( \sqrt{4(\nIterations+2)\mig_\nIterations} + \sum_{\iterIdx=1}^\nIterations \sigma_{\iterIdx-1}(\location^*) \right)
		\end{split}
	\end{equation}
	since $\beta_{\iterIdx}(\delta) \leq \beta_{\iterIdx+1}(\delta)$ for all $\iterIdx\geq 1$. This concludes the proof.
\end{proof}

\subsection{Proof of Theorem 3}
Again, we will be following standard GP-UCB proofs for this result using the bound in Theorem 1.

\begin{proof}[Proof of Theorem 3]
	Extending the bound in \autoref{eq:classifier-regret} with Theorem 1, we have with probability at least $1-\delta$:
	\begin{equation}
		\begin{split}
			\forall\iterIdx\geq 1,\quad \regret_\iterIdx &\leq \Lipschitz_\obsNoise (\hat\classifier_{\iterIdx-1}(\location^*) + \beta_{\iterIdx-1}(\delta)\sigma_{\iterIdx-1}(\location^*) - \classifier_{\iterIdx-1}^*(\location_\iterIdx))\\
			&\leq \Lipschitz_\obsNoise (\hat\classifier_{\iterIdx-1}(\location^*) + \beta_{\iterIdx-1}(\delta)\sigma_{\iterIdx-1}(\location^*) - \hat\classifier_{\iterIdx-1}(\location_\iterIdx) + \beta_{\iterIdx-1}(\delta)\sigma_{\iterIdx-1}(\location_\iterIdx))\\
			&\leq 2\Lipschitz_\obsNoise\beta_{\iterIdx-1}(\delta)\sigma_{\iterIdx-1}(\location_\iterIdx),
		\end{split}
	\end{equation}
	since $\hat\classifier_{\iterIdx-1}(\location^*) + \beta_{\iterIdx-1}(\delta)\sigma_{\iterIdx-1}(\location^*) \leq \max_{\location\in\locDomain} \hat\classifier_{\iterIdx-1}(\location) + \beta_{\iterIdx-1}(\delta)\sigma_{\iterIdx-1}(\location) = \hat\classifier_{\iterIdx-1}(\location_\iterIdx) + \beta_{\iterIdx-1}(\delta)\sigma_{\iterIdx-1}(\location_\iterIdx))$. Turning our attention to the cumulative regret, with the same probability, we have:
	\begin{equation}
		\begin{split}
			\Regret_\nIterations := \sum_{\iterIdx=1}^\nIterations \regret_\iterIdx &\leq 2\Lipschitz_\obsNoise \beta_\nIterations(\delta)\sum_{\iterIdx=1}^\nIterations \sigma_{\iterIdx-1}(\location_\iterIdx)\\
			&\leq 4\Lipschitz_\obsNoise \beta_\nIterations(\delta) \sqrt{(\nIterations+2)\mig_\nIterations},
		\end{split}
	\end{equation}
	which concludes the proof.
\end{proof}

\subsection{Proof of Theorem 4}
%To prove regret bounds for batch BORE++, we consider an idealised setting in the many-particle limit of SVGD $\nSamples\to\infty$. In this case, we are trying to solve approximate inference with SVGD. According to \citet{Liu2017}, under mild conditions, SVGD asymptotically converges to the target distribution in the many-particle limit, which then allows us to assume equivalence between $\variational_\iterIdx$ and $\hat p_\iterIdx$. 
%In this setting, we can apply similar steps to \citet{Oliveira2021} to derive a distributional regret bound. Note that the mode of $\lowerpdf \propto \classifier$ coincides with the minimiser of $\objective$ (see Lemma 1). Therefore, the regret definition above measures how well the algorithm performs with respect to an optimal sampling distribution.

\begin{proof}
	Starting with the regret definition, we can define a bound in terms of the discrepancy between the two sampling distributions:
	\begin{equation}
		\begin{split}
			\regret_\iterIdx &\coloneqq \expectation_{\location \sim \hat{p}_\iterIdx}[\objective(\location)] - \expectation_{\location\sim \lowerpdf}[\objective(\location)]\\
			&\leq \Lipschitz_\obsNoise\left(\expectation_{\location\sim \lowerpdf}[\classifier(\location)] - \expectation_{\location \sim \hat{p}_\iterIdx}[\classifier(\location)]\right)\\
			&\leq \Lipschitz_\obsNoise \norm{\classifier}_\infty \int_\locDomain \lvert \lowerpdf(\location) - \variational_{\iterIdx-1}(\location)\rvert \diff\location\\
			&\leq \Lipschitz_\obsNoise \norm{\classifier}_\infty \sqrt{\frac{1}{2}\kl{\variational_{\iterIdx-1}}{\lowerpdf}}, \qquad \forall \iterIdx\geq 1,
		\end{split}
	\end{equation}
	where the last line is due to Pinsker's inequality \citep{Boucheron2013} applied to the total variation distance between $\hat p_\iterIdx$ and $\lowerpdf$ (third line).
	
	Tp obtain a bound on $\kl{\hat p_\iterIdx}{\lowerpdf}$, starting from the definition of the terms, with probability at least $1-\delta$, we have that:
	\begin{equation}
		\begin{split}
			\forall\iterIdx\geq 0, \quad \kl{\hat p_{\iterIdx}}{\lowerpdf} &= \expectation_{\location\sim\hat p_\iterIdx}[\log \hat p_\iterIdx(\location) - \log \lowerpdf(\location)]\\
			&= \expectation_{\location\sim\hat p_\iterIdx}[\log(\hat\classifier_\iterIdx(\location) + \beta_\iterIdx(\delta)\sigma_\iterIdx(\location)) - \log\classifier(\location) + \log\normalisation_\classifier - \log\percentile]\\
			&\leq \expectation_{\location\sim\hat p_\iterIdx}[\log(\hat\classifier_\iterIdx(\location) + \beta_\iterIdx(\delta)\sigma_\iterIdx(\location)) - \log\classifier(\location)]\,,
		\end{split}
	\end{equation}
	which follows from $\normalisation_\iterIdx := \int_\locDomain (\hat{\classifier}_\iterIdx(\location)+\beta_\iterIdx(\delta)\sigma_\iterIdx(\location))p(\location)\diff\location \geq \int_\locDomain\classifier(\location)p(\location)\diff\location =: \percentile$. Now, by the mean value theorem \citep{Munkres1975}, for all $\iterIdx\geq 0$, we have that the following holds with the same probability:
	\begin{equation}
		\begin{split}
			|\log(\hat\classifier_\iterIdx(\location) + \beta_\iterIdx(\delta)\sigma_\iterIdx(\location)) - \log\classifier(\location)|
			%&\leq \max_{\anyscalar\in(\classifier(\location), \hat\classifier_\iterIdx(\location)+\beta_\iterIdx(\delta)\sigma_\iterIdx(\location))} \frac{\diff\log(\anyscalar)}{\diff\anyscalar} |\hat\classifier_\iterIdx(\location) + \beta_\iterIdx(\delta)\sigma_\iterIdx(\location) - \classifier(\location)|\\
			&\leq \Lipschitz_\classifier|\hat\classifier_\iterIdx(\location) + \beta_\iterIdx(\delta)\sigma_\iterIdx(\location) - \classifier(\location)|\\
			&\leq 2\Lipschitz_\classifier\beta_\iterIdx(\delta)\sigma_\iterIdx(\location)\,, \quad \forall\location\in\locDomain\,,
		\end{split}
	\end{equation}
	since $\frac{\diff\log(\anyscalar)}{\diff\anyscalar} < \Lipschitz_\classifier < \infty$ for all $\anyscalar > \min_{\location\in\locDomain}\classifier(\location) > 0$, and $|\hat\classifier_\iterIdx(\location) - \classifier(\location)| \leq \beta_\iterIdx(\delta)\sigma_\iterIdx(\location)$ by Theorem 1. The first result in the theorem then follows.

	For the second part of the result, we first note that:
	\begin{equation}
		\forall \nIterations\geq 1, \quad \min_{\iterIdx\leq \nIterations} \kl{\hat p_\iterIdx}{\lowerpdf} \leq \frac{1}{\nIterations} \sum_{\iterIdx=1}^\nIterations \kl{\hat p_{\iterIdx-1}}{\lowerpdf}
	\end{equation}
	Following the previous derivations, it holds with probability at least $1-\delta$ that:
	\begin{equation}
		\begin{split}
			\sum_{\iterIdx=1}^\nIterations \kl{\hat p_\iterIdx}{\lowerpdf}
			&\leq 2\Lipschitz_\classifier\sum_{\iterIdx=1}^\nIterations \beta_{\iterIdx-1}(\delta)\expectation_{\vec{\tilde\location}_\iterIdx\sim \hat p_\iterIdx}[\sigma_{\iterIdx-1}(\vec{\tilde\location}_\iterIdx)]\\
			&\leq 2\Lipschitz_\classifier\beta_\nIterations(\delta)\sum_{\iterIdx=1}^\nIterations\expectation_{\vec{\tilde\location}_\iterIdx\sim q_\iterIdx}[\sigma_{\iterIdx-1}(\vec{\tilde\location}_\iterIdx)]\\
			&\leq 2\Lipschitz_\classifier\beta_\nIterations(\delta)\expectation_{\vec{\tilde\location}_1\sim q_1,\dots,\vec{\tilde\location}_\nIterations\sim q_\nIterations}\left[\sum_{\iterIdx=1}^\nIterations\sigma_{\iterIdx-1}(\vec{\tilde\location}_\iterIdx)\right]\,,
		\end{split}
		\label{eq:regret-sum-variance}
	\end{equation}
	since $\beta_\nIterations \geq \beta_{\iterIdx}$, for all $\iterIdx\leq \nIterations$, and expectations are linear operations. Considering the predictive variances above, recall that, at each iteration $\iterIdx\geq 1$, the algorithm selects a batch of \iid points $\set{B}_\iterIdx := \{\location_{\iterIdx,i}\}_{i=1}^\nSamples$, sampled from $\hat p_\iterIdx$, where to evaluate the objective function $\objective$. The predictive variance $\sigma_{\iterIdx-1}^2$ is conditioned on all previous observations, which are grouped by batches. We can then decompose, for any $\iterIdx\geq 1$:
	\begin{equation}
		\begin{split}
			\sigma_{\iterIdx}^2(\location) &= \sigma_{\iterIdx-1}^2(\location) - k_{\iterIdx-1}(\location,\set{B}_{\iterIdx})(\mat k_{\iterIdx-1}(\set{B}_{\iterIdx}) + \eta\eye)^{-1}k_{\iterIdx-1}(\set{B}_{\iterIdx}, \location)\,,
		\end{split}
	\end{equation}
	where we use the notation introduced in \autoref{thr:batch-variance}, and:
	\begin{align}
		\begin{split}
			k_{\iterIdx}(\location,\location') &= k_{\iterIdx-1}(\location,\location') - k_{\iterIdx-1}(\location, \set{B}_{\iterIdx})(\mat k_{\iterIdx-1}(\set{B}_{\iterIdx}) + \eta\eye)^{-1}k_{\iterIdx-1}(\set{B}_{\iterIdx}, \location'), \quad \iterIdx\geq 1,
		\end{split}\\
		k_0(\location,\location') &:= k(\location,\location')\,.
	\end{align}
	Therefore, the predictive variance of the batched algorithm is not the same as the predictive variance of a sequential algorithm, and we cannot direcly apply \autoref{thr:variance-sum-bound} to bound the last term in \autoref{eq:regret-sum-variance}.
	
	\autoref{thr:batch-variance} tells us that the predictive variance given a set of observations is less than the predictive variance given a subset of observations. Selecting only the first point from within each batch and applying \autoref{thr:batch-variance}, we get, for $\iterIdx\geq 1$:
	\begin{equation}
		\begin{split}
			\sigma^2_\iterIdx(\location) &\leq s_\iterIdx^2(\location):= k(\location,\location) - k(\location, \locDomain_\iterIdx)(\mat k(\locDomain_\iterIdx) + \eta\eye)^{-1}k(\locDomain_\iterIdx, \location)\,,
		\end{split}
	\end{equation}
	where $\locDomain_\iterIdx := \{\location_{i,1}\}_{i=1}^\iterIdx$, with $\location_{i,1} \in \set{B}_i$, $i\in\{1,\dots,\iterIdx\}$. Note that the right-hand side of the equation above is simply the non-batched GP predictive variance. Furthermore, sample points within a batch are \iid, so that $\location_{\iterIdx,1}\sim q_\iterIdx$ and $\vec{\tilde{\location}}_{\iterIdx}\sim q_\iterIdx$ are identically distributed. We can now apply \autoref{thr:variance-sum-bound}, yielding:
	\begin{equation}
		\begin{split}
			\expectation_{\vec{\tilde\location}_1\sim q_1,\dots,\vec{\tilde\location}_\nIterations\sim q_\nIterations}\left[\sum_{\iterIdx=1}^\nIterations\sigma_{\iterIdx-1}(\vec{\tilde\location}_\iterIdx)\right] \leq  \expectation_{\vec{\tilde\location}_1\sim q_1,\dots,\vec{\tilde\location}_\nIterations\sim q_\nIterations}\left[ \sum_{\iterIdx=1}^\nIterations s_{\iterIdx-1}(\vec{\tilde\location}_\iterIdx)\right]\leq 2\sqrt{(\nIterations+2)\mig_\nIterations}\,.
		\end{split}
	\end{equation}
	Combining this result with \autoref{eq:regret-sum-variance}, we obtain:
	\begin{equation}
		\sum_{\iterIdx=1}^\nIterations \kl{\hat p_\iterIdx}{\lowerpdf} \leq 4\Lipschitz_\classifier\beta_\nIterations(\delta)\sqrt{(\nIterations+2)\mig_\nIterations} \in \set{O}(\beta_\nIterations(\delta)\sqrt{\nIterations\mig_\nIterations})\,.
	\end{equation}
	Lastly, from the definition of $\beta_\iterIdx(\delta)$, we have:
	\begin{equation}
		\beta_\nIterations(\delta) := \bound + \sqrt{2\regFactor^{-1}\log(|\eye+\regFactor^{-1}\mat K_{\dataset_{\nIterations}}|^{1/2}/\delta)}\,,
	\end{equation}
	where:
	\begin{equation}
		\log(|\eye+\regFactor^{-1}\mat K_{\dataset_{\nIterations}}|^{1/2}) = \mutinfo{\vec\clabel_{\nObs_\nIterations}, \vec \gpfunction_{\nObs_\nIterations}} \leq \mig_{\nObs_\nIterations} = \mig_{\nSamples\nIterations}\,,
	\end{equation}
	for $\gpfunction\sim \gp(\gpMeanFunction,k)$. Therefore, the cumulative sum of divergences is such that:
	\begin{equation}
		\sum_{\iterIdx=1}^\nIterations \kl{\hat p_\iterIdx}{\lowerpdf} \in \set{O}(\sqrt{\nIterations}(b\sqrt{\mig_\nIterations}+\sqrt{\mig_\nIterations\mig_{\nSamples\nIterations}}))\,.
	\end{equation}
	which concludes the proof.
\end{proof}

\section{Bayesian regret bounds for BORE as Thompson sampling}
\label{sec:bore-thompson}
Although in our main results we considered BORE using an optimal classifier according to a least-squares loss, we may instead consider that, in practice, the trained classifier might be sub-optimal due to training via gradient descent. In particular, in the case of stochastic gradient descent, \citet{Mandt2017} showed that parameters learnt this way can be seen as approximate samples of a Bayesian posterior distribution. This is, therefore, the case of Thompson (or posterior) sampling \citep{Russo2016}. If we consider that the posterior over the model's function space is Gaussian, e.g., as in the case of infinitely-wide deep neural networks \citep{Jacot2018, Matthews2018}, we may instead analyse the original BORE as a GP-based Thompson sampling algorithm. We can then apply theoretical results from \citet{Russo2016} to use general GP-UCB approximation guarantees \citep{Srinivas2010, Grunewalder2010} to bound BORE'S Bayesian regret. Note, however, that this is a different type of analysis compared to the one presented in this paper, which considered a frequentist setting where the objective function is fixed, but unknown.

%TODO: Add section on the issues and maybe some preliminary work on bounding the regret for the logistic-regression loss.

\section{Analysis with a non-constant quantile approximation}
\label{sec:non-constant}
Our main theoretical results so far relied upon the quantile $\quantile$ being fixed throughout all iterations $\iterIdx\in\{1,\dots\nIterations\}$, though in practice we have to approximate the quantile based on the empirical observations distribution up to time $\iterIdx\geq 1$. In this section, we discuss the plausibility of the theoretical results under this practical scenario.

The main impact of a time-varying quantile $\quantile_\iterIdx$, and the corresponding classifier $\classifier_\iterIdx(\location) := p(\observation\leq\quantile_\iterIdx|\location)$, in theoretical results is in the UCB approximation error (Theorem 1). This result depends on the observation noise $\gpnoise_{\iterIdx, i} := \clabel_{\iterIdx, i} - \classifier_\iterIdx(\location_i)$ as perceived by a GP model with observations $\clabel_{\iterIdx, i} := \indicator[\observation_i\leq\quantile_\iterIdx]$, $i\in\{1,\dots, \iterIdx\}$, to be sub-Gaussian when conditioned on the history. Hence, a few challenges originate from there. Firstly, the past observations in the vector $\vec\clabel_\iterIdx := [\clabel_{\iterIdx, i}]_{i=1}^\iterIdx$ are changing across iterations, due to the update in $\quantile_\iterIdx$. Secondly, the latent function $\classifier_\iterIdx$ is stochastic, as the quantile $\quantile_\iterIdx$ depends on the current set of observations $\observations_\iterIdx$. Lastly, it is not very clear how to define a filtration for the resulting stochastic process such that the GP noise $\gpnoise_{\iterIdx, i}$ is sub-Gaussian. Nevertheless, as the number of observations increases, $\quantile$ converges to a fixed value, making our asymptotic results valid.

\section{Experiment details}
\label{sec:exp-details}
This section presents details of our experiments setup. We used PyTorch \citep{PyTorch2019} for our implementation of batch BORE and BORE++, which we plan to make publicly available in the future.

\subsection{Theory assessment}
For this experiment, we generated a random classifier as an element of the RKHS $\Hspace$ defined by the kernel $k$ as:
\begin{equation}
	\classifier^* := \sum_{i=1}^\nFeatures \alpha_i k(\cdot, \location_i^*) \in \Hspace\,,
\end{equation}
where $\{\location_i^*\}_{i=1}^\nFeatures$ and the weights $\{\alpha_i\}_{i=1}^\nFeatures$ were \iid sampled from a unit uniform distribution $\uniform(0,1)$,  with $\nFeatures := 5$. The norm of $\classifier^*$ is given by:
\begin{equation}
	\norm{\classifier^*}_k = \sqrt{\vec\alpha^\transpose_\nFeatures \mat K_\nFeatures \vec\alpha_\nFeatures}\,,
\end{equation}
where $\mat k := [\location_i^*, \location_j^*]_{i,j=1}^\nFeatures \in \R^{\nFeatures\times\nFeatures}$ and $\vec\alpha_\nFeatures := [\alpha_1, \dots, \alpha_\nFeatures]^\transpose \in \R^\nFeatures$. To ensure $\classifier^*(\location) \leq 1$, we normalised the weights according to the classifier norm, i.e., $\vec{\alpha} := \frac{1}{\norm{\classifier^*}}\vec{\alpha}$, so that $\norm{\classifier^*} = 1$, and consequently $\classifier^*(\location) \leq k(\location, \location)\norm{\classifier^*} = \norm{\classifier^*} = 1$, for all $\location\in\locDomain$. The kernel was set as the squared exponential (RBF) with a length-scale of $0.1$.

Given a threshold $\quantile\in\R$, the objective function corresponding to $\classifier^*$ satisfies:
\begin{equation}
	\objective(\location) := \quantile - \cdf_\obsNoise^{-1}(\classifier^*(\location)), \quad \forall\location \in \locDomain\,.
\end{equation}
For this experiment, we set $\quantile := 0$. Each trial had a different objective function generated for it. An example of classifier and objective function pair is presented in Figure 1b (main paper). Observations were composed as function evaluations corrupted by zero-mean Gaussian noise with variance $\sigma_\obsNoise^2 := 0.01$.

The search space was configured as a finite set $\locDomain := \{\location_i\}_{i=1}^{N_\locDomain
}\subset [0, 1]$ by sampling $N_\locDomain$ points from a unit uniform distribution. The number of points in the search space was set as $N_\locDomain := 100$. As the search space is finite, we also know $\percentile := p(\observation \leq\quantile) = \int_\locDomain \classifier(\location)p(\location)\diff\location = \frac{1}{N_\locDomain}\sum_{\location\in\locDomain}\classifier^*(\location)$.

\begin{table}
	\centering
	\begin{tabular}{c|c}
		\textbf{Parameter} & \textbf{Setting}\\
		\hline
		$\regFactor$ & 0.025\\
		$\delta$ & 0.1\\
		$\quantile$ & 0\\
		\hline
	\end{tabular}
	\caption{Settings for BORE++ in the theory assessment experiment.}
	\label{tab:pls-params}
\end{table}

\begin{table}
	\centering
	\begin{tabular}{c|c}
		\textbf{Parameter} & \textbf{Setting}\\
		\hline
		$\regFactor$ & 0.01\\
		$\delta$ & 0.1\\
		$\sigma_\obsNoise^2$ & 0.01\\
		\hline
	\end{tabular}
	\caption{Settings for GP-UCB in the theory assessment experiment.}
	\label{tab:gpucb-params}
\end{table}

Regarding algorithm hyper-parameters, although any upper bound $\bound \geq \norm{\classifier^*}$ would work for setting up $\beta_\iterIdx$, BORE++ was configured with the RKHS norm $\classifier^*$ as the first term in the setting for $\beta_\iterIdx$ (see Theorem 1). To configure GP-UCB according to its theoretical settings \citep[Thm. 1]{Durand2018}, we computed the RKHS norm of the resulting $\objective$ in the RKHS. We can compute the norm of $\objective$ as an element of $\Hspace$ by solving the following constrained optimisation problem:
\begin{equation}
	\begin{split}
		\norm{\objective}_k &= \min_{\anyfunction\in\Hspace} \norm{\anyfunction}_k, \quad \mathrm{s.t.}\quad \anyfunction(\location) = \objective(\location), \quad \forall\location\in\locDomain\,.
	\end{split}
\end{equation}
As the search space is finite, the solution to this problem is available in closed form as:
\begin{equation}
	\norm{\objective}_k = \sqrt{\vec\objective^\transpose_\locDomain\mat k_\locDomain^{-1} \vec\objective_\locDomain}\,,
\end{equation}
where $\vec\objective_\locDomain := [\objective(\location)]_{\location\in\locDomain}\in\R^{N_\locDomain}$, and $\mat k_\locDomain := [k(\location, \location')]_{\location, \location'\in\locDomain}$. We set $\delta := 0.1$. For both BORE++ and GP-UCB, the information gain was recomputed at each iteration. Lastly, the regularisation factor $\regFactor$ was set as $\regFactor := \sigma_\obsNoise^2$ for GP-UCB and as $\regFactor := 0.025$ for BORE++, which was found by grid search. {In summary, for this experiment, the settings for BORE++ can be found in \autoref{tab:pls-params} and, for GP-UCB, in \autoref{tab:gpucb-params}.}

\subsection{Global optimisation benchmarks}
For each problem, all methods used 10 initial points uniformly sampled from the search space. As performance indicator, we measured the simple regret:
\begin{equation}
	\regret_\iterIdx^* := \min_{i\leq \iterIdx} \regret_i = \min_{i\leq\iterIdx} \objective(\location_i) -  \objective(\location^*)\,, \quad \iterIdx\geq 1\,,
\end{equation}
as the global minimum of each of the considered benchmark functions is known. All objective function evaluations were provided free of noise to the algorithms.

\begin{table}[t]
	\centering
	\begin{tabular}{c|c}
		\textbf{Parameter} & \textbf{Setting}\\
		\hline
		Optimiser & Adam\\
		Batch size & 64\\
		Steps & 100\textsuperscript{*}\\
		\hline
	\end{tabular}
	\caption{Stochastic gradient descent training settings for batch BORE. (*) For the Six-hump Camel problem, we applied 200 steps.}
	\label{tab:training}
\end{table}

\begin{table}[t]
	\centering
	\begin{tabular}{c|c}
		\textbf{Parameter} & \textbf{Setting} \\
		\hline
		Step size & 0.001\\
		Decay rate & 0.9\\
		Number of steps & $1000^*$\\
		\hline
	\end{tabular}
	\caption{SVGD settings for batch BORE. (*) For the Hartmann 3D problem, we used 500 steps.}
	\label{tab:svgd}
\end{table}

Batch BORE was run with a percentile $\percentile := 0.25$, which was applied to estimate the empirical quantile $\quantile$ at every iteration $\iterIdx\in\{1,\dots,\nIterations\}$. The method's classifier model was composed of a multilayer perceptron neural network model with 2 hidden layers of 32 units each, which was trained to minimise the binary cross-entropy loss. The activation function was set as the rectified linear unit (ReLU) with exception for the Hartmann 3D and the Six-hump Camel problem, which were run with an exponential linear unit (ELU), instead. Training for the neural networks was performed via stochastic gradient descent, whose settings are presented in \autoref{tab:training}. SVGD was run applying Adadelta to configure its steps according to the settings in \autoref{tab:svgd}. The SVGD kernel was set as the squared exponential (RBF) using the median trick to adjust its lengthscale \citep{Liu2016}.

LP-EI \citep{Gonzalez2016batch} was run using L-BFGS \citep{Byrd1995limited} to optimise its acquisition function. The optimisation settings were kept as the default for GPyOpt \citep{GPyOpt2016}.

The $q$-EI method \citep{Snoek2012} was run using the BoTorch implementation \citep{Balandat2020botorch}. The acquisition function was optimised via multi-start optimisation with L-BFGS \citep{Byrd1995limited} using 10 random restarts. Monte Carlo integration for $q$-EI used 256 samples.

%TODO: Double-check batch size and other settings.

\subsection{Comparisons on real-data benchmarks}
We here present experiments comparing the sequential version of BORE++ against BORE and other baselines, including traditional BO methods, such as GP-UCB and GP-EI \citep{Shahriari2016}, the Tree-structured Parzen Estimator (TPE) \citep{bergstra2011algorithms}, and random search, on real-data benchmarks. In particular, we assessed the algorithms on some of the same benchmarks present in the original BORE paper \citep{tiao2021bore}.

\subsubsection{Algorithm settings}
All versions of BORE were set with $\gamma := 0.25$. The original BORE algorithm used a 2-layer, 32-unit fully connected neural network as a classifier. The network was trained via stochastic gradient descent using Adam \citep{Kingma2015}. As in the other experiments in this paper, we followed the same scheme that keeps the number of gradient steps per epoch fixed \citep[see][Appendix J.3]{tiao2021bore}, set in our case as 100, and a mini-batch of size 64. The probabilistic least-squares version of BORE and BORE++ were configured with a GP classifier using the rational quadratic kernel \citep[Ch. 4]{Rasmussen2006} with output scaling and independent length scales per input dimension. All GP-based algorithms used the same type of kernel. GP hyper-parameters were estimated by maximising the GP's marginal likelihood at each iteration using BoTorch's hyper-parameter estimation methods, which apply L-BFGS by default \citep{Balandat2020botorch}. BORE++ was set with a fixed value for its parameter $\beta_\iterIdx := 3$, the regularisation factor was set as $\regFactor:=0.025$. Acquisition function optimisation was run for 500 to 1000 iterations via L-BFGS with multiple restarts using SciPy's toolkit \citep{SciPy2020}. Lastly, for the experiment with the MNIST dataset, we also used the Tree-structured Parzen Estimator (TPE) by \citet{bergstra2011algorithms} set with default settings from the HyperOpt package. All algorithms were run for a given number of independent trials and results are presented with their 95\% confidence intervals\footnote{Confidence intervals are calculated via linear interpolation when the number of trials is small.}

\subsubsection{Benchmarks}
\paragraph{Neural network hyper-parameter tuning.}
We first considered two of the neural network (NN) tuning problems found in \citet{tiao2021bore}, where a two-layer feed-forward NN is trained for regression. The NN is trained for 100 epochs with the ADAM optimizer \citep{Kingma2015}, and the objective is the validation mean-squared error (MSE). The hyper-parameters are the initial learning rate, learning rate schedule, batch size, along with the layer-specific widths, activations, and dropout rates. In particular, we considered Parkinson's telemonitoring \citep{Tsanas2009Parkinsons} and the CT slice localisation \citep{Graf2011CTslice} datasets, available at UCI's machine learning repository \citep{Dua2019}, and utilize HPOBench \citep{Eggensperger2021hpobench}, which tabulates, for each dataset, the MSEs resulting from all possible (62,208) configurations. The datasets and code are publicly available\footnote{Tabular benchmarks: \url{https://github.com/automl/nas_benchmarks}}. Each algorithm was run for 500 iterations across 10 independent trials.

\paragraph{Racing line optimisation.} We compare the different versions of BORE against a random search baseline in the UC Berkeley racing line optimisation benchmark \citep{Liniger2015} using the code provided by \citet{JainRaceOpt2020}. The task consists of finding the optimal racing line across a given track by optimising the positions of a set of 10 waypoints on the Cartesian plane along the track's centre line which would reduce the lap time for a given car, resulting in a 20-dimensional problem. For this track, the car model is based on UC Berkeley's 1:10 scale miniature racing car open source model\footnote{Open source race car: \url{https://github.com/MPC-Berkeley/barc/tree/devel-ugo}}. Each algorithm was run for 50 iterations across 5 independent trials.

\paragraph{Neural architecture search.} Lastly, we compare all algorithms on a neural network architecture search problem. The task consists of optimising hyper-parameters which control the training process (initial learning rate, batch size, dropout, exponential decay factor for learning rate) and the architecture (number of layers and units per layer) of a feed forward neural network on the MNIST hand-written digits classification task \citep{Deng2012mnist}. The objective is to minimise the NN classification error. To allow for a wide range of hyper-parameter evaluations, this task uses a random forest surrogate trained with data obtained by training the actual NNs on MNIST \citep{Falkner18a}. For this experiment, each method was run for 200 iterations across 10 independent trials.

\section{Dimensionality effect on batch BORE vs. batch BORE++ performance}
\label{sec:extra-exp}

\begin{figure}
	\centering
	\subfloat[3D]{\includegraphics[width=0.33\linewidth]{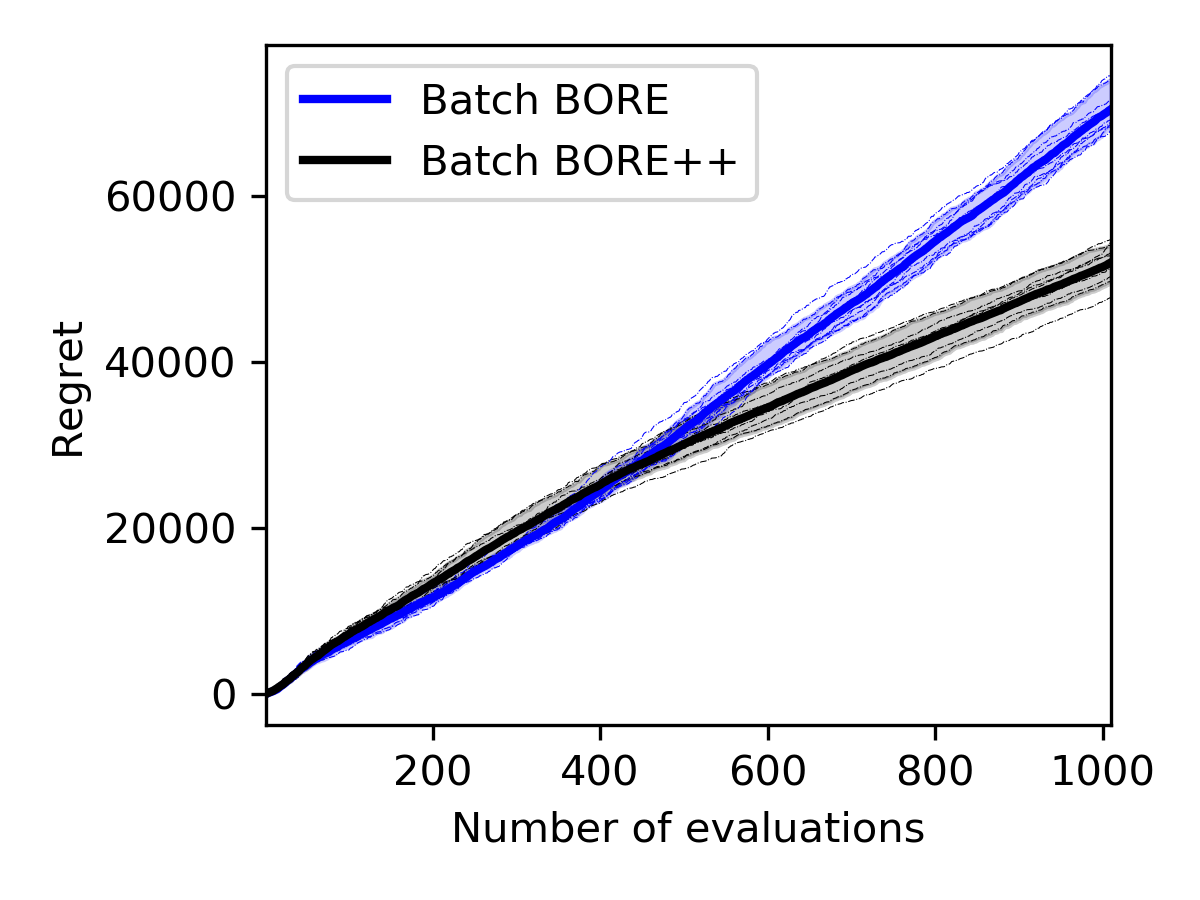}}
	\subfloat[4D]{\includegraphics[width=0.33\linewidth]{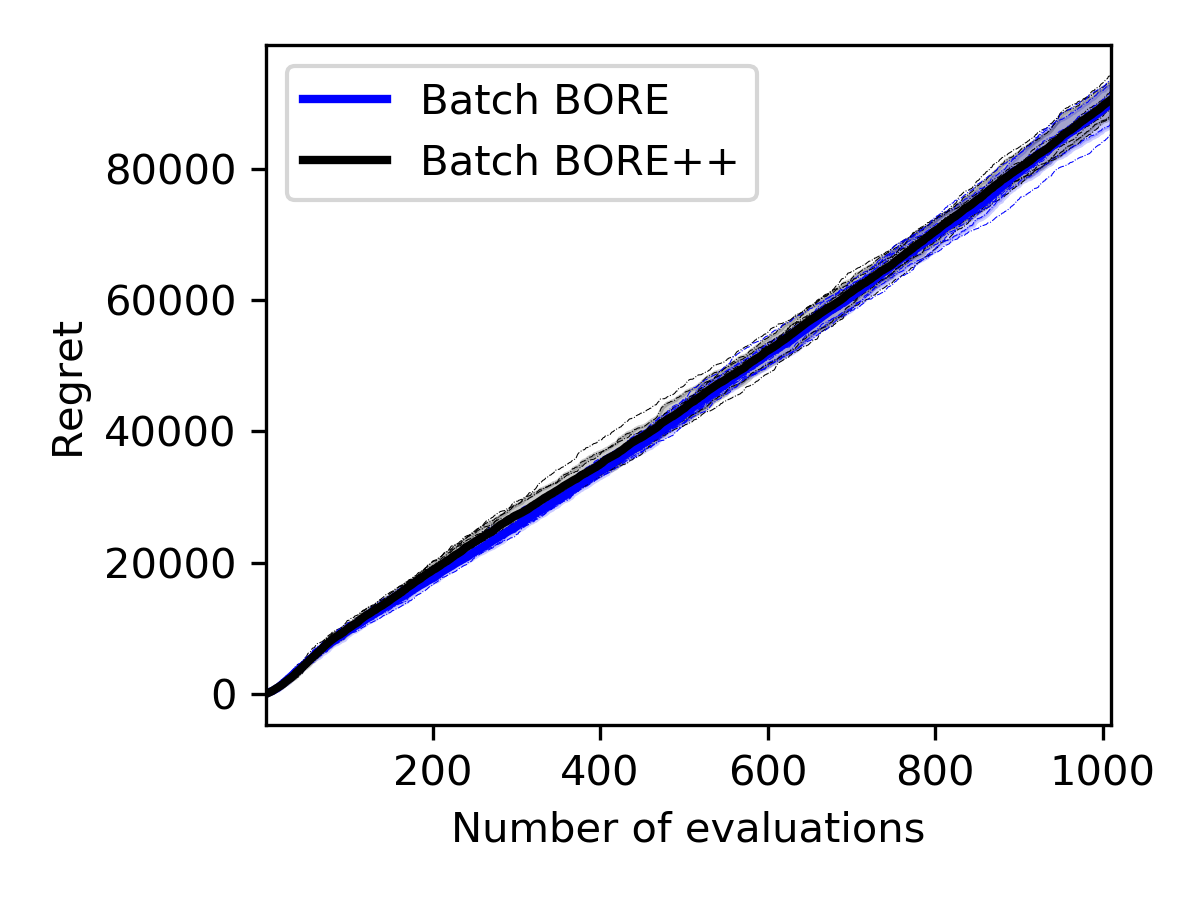}}
	\subfloat[5D]{\includegraphics[width=0.33\linewidth]{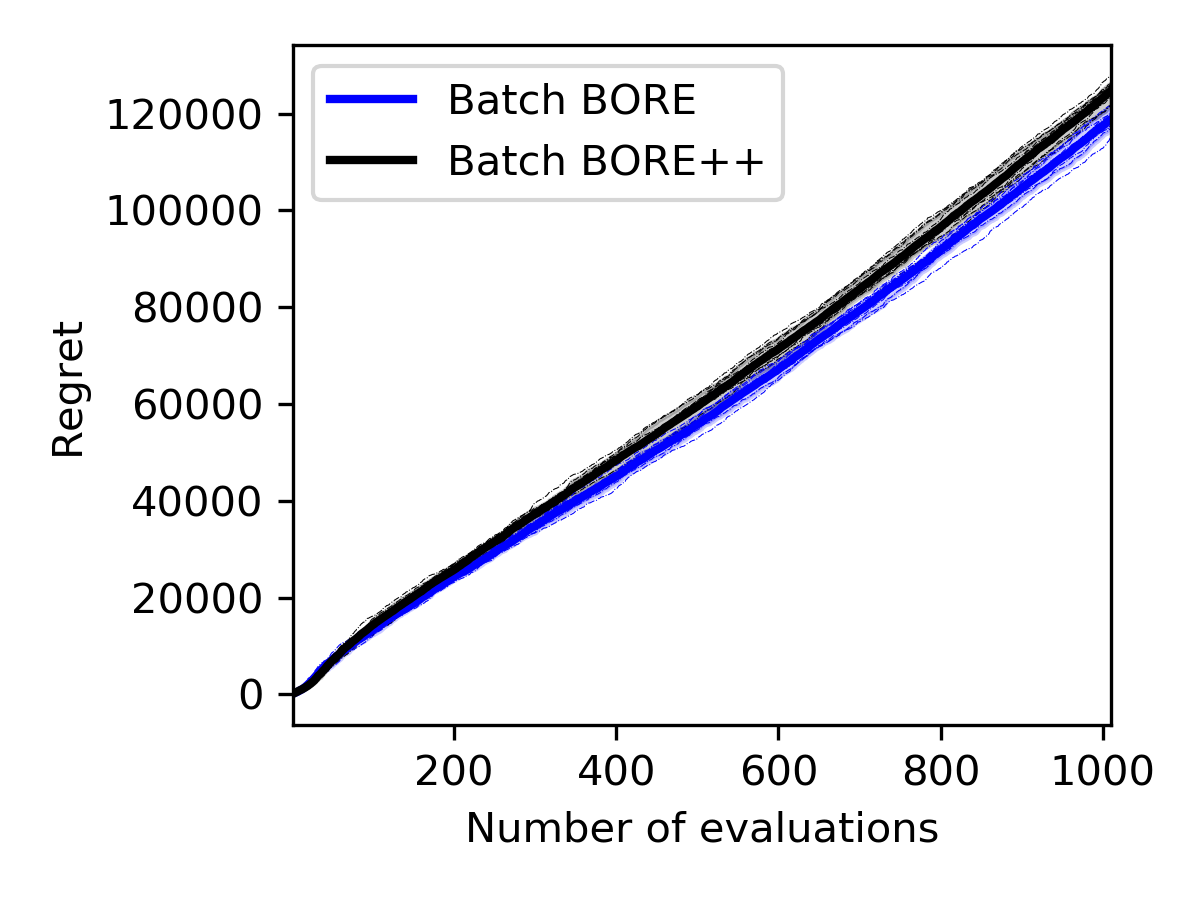}}
	\caption{BORE vs. BORE++ in the batch setting tested on the Rosenbrock function at varying search space dimensionalities. The plots compare the cumulative regret of each algorithm averaged over 10 runs. Shaded areas correspond to the 95\% confidence interval.}
	\label{fig:batch-comparison}
\end{figure}

We compared batch BORE against the batch BORE++ algorithm on a synthetic optimisation problem with the Rosenbrock function. The dimensionality of the search space was varied. The cumulative regret curves for each algorithm are presented in \autoref{fig:batch-comparison}.

Both algorithms were configured with a Bayesian logistic regression classifier applying random Fourier features \citep{Rahimi2007} as feature maps based on the squared-exponential kernel. The number of features was set as 300, and the classifier was trained via expectation maximisation. Observations were corrupted by additive Gaussian noise with zero mean and a small noise variance $\sigma_\obsNoise^2 = 10^{-4}$, and each model was set accordingly. To demonstrate the practicality of the method, the UCB parameter for BORE++ was fixed at $\beta_\iterIdx := 3$ across all iterations $\iterIdx\geq 1$, instead of applying the theoretical setup. SVGD was configured as its second-order version \citep{Detommaso2018} applying L-BFGS to adjust its steps \citep{Byrd1995limited}.

As the results show in \autoref{fig:batch-comparison}, batch BORE++ has a clear advantage over batch BORE in low dimensions. However, the performance gains become less obvious at higher dimensionalities and eventually deteriorate. One of the factors explaining this behaviour is that, as the dimensionality increases, uncertainty estimates become less useful. Distances between data points largely increase and affect the posterior variance estimates provided by translation-invariant kernels, such as the squared-exponential kernel our feature maps were based on. Other classification models may lead to different behaviours, and their investigation is left for future work.

\end{document}